\documentclass{article}

\usepackage[final]{neurips_2024}

\usepackage{titletoc}

\usepackage[utf8]{inputenc} 
\usepackage[T1]{fontenc}    
\usepackage{hyperref}       
\usepackage{url}            
\usepackage{booktabs}       
\usepackage{amsfonts}       
\usepackage{nicefrac}       
\usepackage{microtype}      

\usepackage{mathrsfs}
\usepackage{graphicx}
\usepackage{subfigure}
\usepackage{multirow}
\usepackage{amsmath}
\usepackage{float}
\usepackage{amsthm}
\usepackage{comment}
\usepackage{multicol}
\usepackage{appendix}
\usepackage{amssymb}
\usepackage{wrapfig}
\usepackage{bbm}
\usepackage{bbding}
\usepackage{framed}
\usepackage{setspace}
\usepackage{pifont}

\usepackage[table]{xcolor}
\usepackage{xcolor}
\usepackage{pythonhighlight}

\newcommand{\cmark}{\ding{51}}
\newcommand{\xmark}{\ding{55}}

\usepackage{makecell}
\usepackage[linesnumbered,ruled,lined]{algorithm2e}
\newtheorem{lemma}{Lemma}
\newtheorem{theorem}{Theorem}

\newtheorem{definition}{Definition}
\definecolor{mycolor1}{rgb}{0.82,0.70,0.54}
\definecolor{mycolor2}{rgb}{0.0,0.51,0.22}
\definecolor{mycolor3}{rgb}{0.80, 0.48, 0.37}
\definecolor{mycolor4}{rgb}{0.02, 0.33, 0.68}
\definecolor{mycolor5}{rgb}{0.86, 0.11, 0.11}

\usepackage{textcomp,booktabs}

\title{The Best of Both Worlds: On the Dilemma of Out-of-distribution Detection}
\author{%
  Qingyang~Zhang\thanks{Work done during an internship at Tencent AI Lab.}\\
  Tianjin University\\
  \And
  Qiuxuan~Feng \\
  Tianjin University\\
  \And
  Joey Tianyi Zhou \\
  A*STAR\\
  \And
  Yatao Bian\\Tencent AI Lab\\
  \And
  Qinghua Hu\\Tianjin University\\
  \And
  Changqing Zhang\thanks{Correspondence to Changqing Zhang <zhangchangqing@tju.edu.cn>} \\Tianjin University\\
}

\begin{document}

\maketitle

\begin{abstract}
Out-of-distribution (OOD) detection is essential for model trustworthiness which aims to sensitively identify semantic OOD samples and robustly generalize for covariate-shifted OOD samples. However, we discover that the superior OOD detection performance of state-of-the-art methods is achieved by secretly sacrificing the OOD generalization ability. Specifically, the classification accuracy of these models could deteriorate dramatically when they encounter even minor noise. This phenomenon contradicts the goal of model trustworthiness and severely restricts their applicability in real-world scenarios. What is the hidden reason behind such a limitation? In this work, we theoretically demystify the ``\textit{sensitive-robust}'' dilemma that lies in many existing OOD detection methods. Consequently, a theory-inspired algorithm is induced to overcome such a dilemma. By decoupling the uncertainty learning objective from a Bayesian perspective, the conflict between OOD detection and OOD generalization is naturally harmonized and a dual-optimal performance could be expected. Empirical studies show that our method achieves superior performance on standard benchmarks. To our best knowledge, this work is the first principled OOD detection method that achieves state-of-the-art OOD detection performance without compromising OOD generalization ability. 
Our code is available at \href{https://github.com/QingyangZhang/DUL}{https://github.com/QingyangZhang/DUL}.
\end{abstract}

\section{Introduction}
Endowing machine learning models with out-of-distribution (OOD) detection and OOD generalization ability are both essential for their deployment in the open world~\cite{park2021reliable,amodei2016concrete,liu2021towards}. We borrow an example of autonomous driving from~\cite{bai2023feed} to demonstrate the motivation of these two tasks. Given a machine learning model trained on in-distribution (ID) data (top image in Fig.~\ref{fig:coverfigure} (a)), OOD detection aims to sensitively perceive uncertainty arising upon outliers that do not belong to any known classes of training data~\cite{zhang2023openood} (bottom right image in Fig.~\ref{fig:coverfigure} (a)). While OOD generalization expects machine learning models to be robust in the presence of unexpected noise or corruption, e.g., rainy or snowy weather (bottom left image in Fig.~\ref{fig:coverfigure} (a)). In this paper, we reveal that many previous methods pursue OOD detection performance at a secret cost of sacrificing OOD generalization ability. 
To make things worse, we observe that some SOTA OOD detection methods may result in a catastrophic collapse in classification performance ($\sim$15\% accuracy degradation) when encountering even slight noise. One pioneering work~\cite{bai2023feed} makes a trade-off between OOD detection and OOD generalization, but the relationship between these two tasks is still largely unexplored.
\begin{figure}[htbp]
    \centering
    \includegraphics[width=0.99\textwidth]{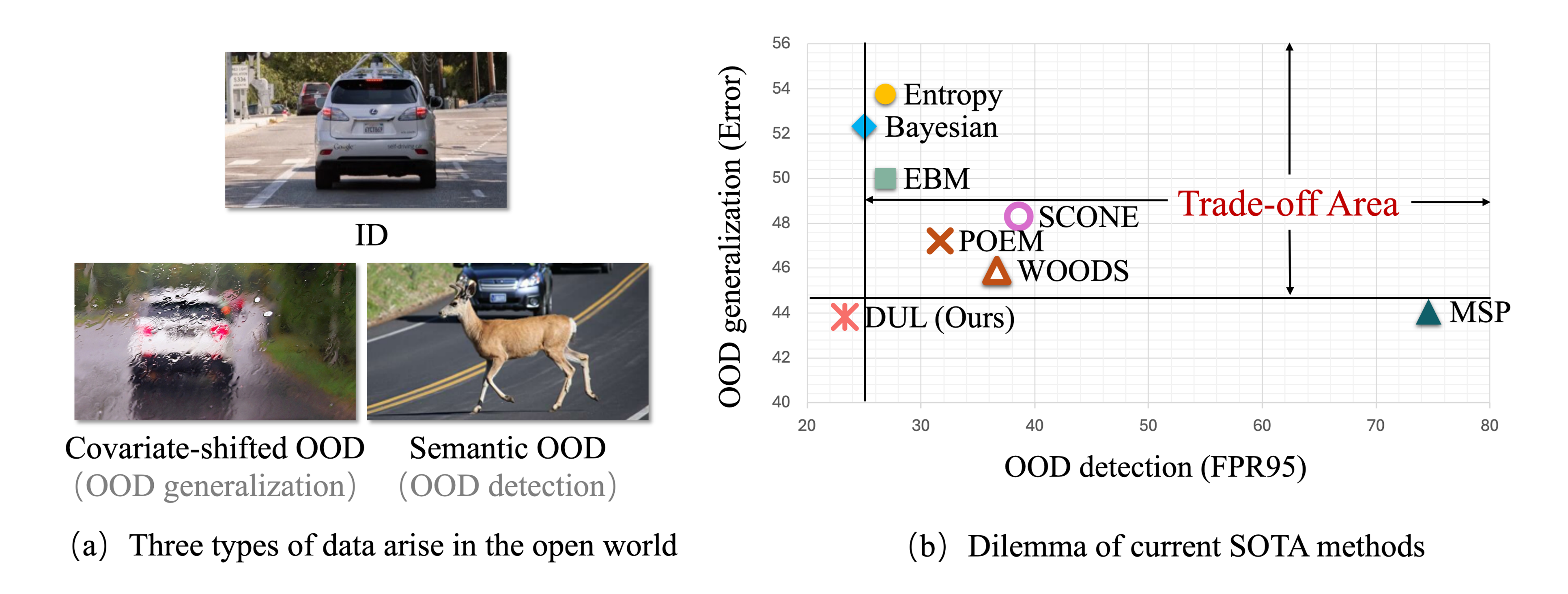}
    \vspace{-6pt}
    \caption{\textbf{(a)}: Models trained on in-distribution (ID) data inevitably encounter distributional shifts during their deployment. OOD generalization expects the model to correctly classify covariate-shifted data that undergoes noise or corruption due to environmental issues. OOD detection aims to identify samples that do not belong to any known classes for trustworthiness consideration. \textbf{(b)}: Limitations of current advanced OOD detection methods. We consider 8 representative OOD detection methods including the baseline method MSP~\cite{hendrycks2016baseline} (without any OOD detection regularization), Entropy~\cite{hendrycks2018deep}, EBM~\cite{liu2020energy}, Bayesian~\cite{malinin2018predictive}, SOTA OOD detection methods WOODS~\cite{katz2022training}, POEM~\cite{ming2022poem}, recent advanced SCONE~\cite{bai2023feed} which aims to seek for a good trade-off and the proposed DUL. All these methods exhibit a degraded generalization ability compared to baseline method MSP and lie in a trade-off area except our DUL. The goal of this paper is to understand and mitigate this phenomenon.}
    \label{fig:coverfigure}
    \vspace{-0.60cm}
\end{figure}
The learning objectives of these two tasks are seemingly conflicting at first glance. OOD detection encourages sensitive uncertainty awareness (highly uncertain prediction) on unseen data, while generalization expects the prediction to be confident and robust under unforeseeable distributional shifts. Previous work in OOD detection research area~\cite{bai2023feed} characterizes the relationship between OOD detection and OOD generalization as a trade-off and thus striking for a balanced performance. However, this trade-off significantly limits the employment of current state-of-the-art OOD detection methods. Naturally, one might require the model to be aware of the OOD input for ensuring safety, but certainly does not expect to sacrifice the generalization ability, not to mention that the catastrophically collapsed classification performance under noise or corruption.

In this work, we first uncover the potential reason behind this limitation by characterizing the generalization error lower bound of previous OOD detection methods, which is referred to \textit{sensitive-robust dilemma}. To overcome the dilemma, we devise a novel Decoupled Uncertainty Learning (DUL) framework for dual-optimal performance. The decoupled uncertainties are separately responsible for characterizing semantic OOD (detection) and covariate-shifted OOD (generalization). Thanks to the decoupled uncertainty learning objective, dual-optimal OOD detection and OOD generalization performance could be expected. Our emphasis lies on a particular category of OOD detection methods in the classification task, including max softmax probability (MSP) based model~\cite{hendrycks2016baseline}, energy-based model (EBM)~\cite{liu2020energy} and Bayesian methods~\cite{malinin2018predictive}. This selection offers two-fold advantages. First, MSP, EBM and Bayesian detectors encompass major OOD detection advances in classification task~\cite{zhang2023openood}. Second, numerous OOD detection works in diverse learning tasks (classification, object detection~\cite{du2022vos}, time-series prediction and image segmentation) are all roughly related to classification~\cite{yang2021generalized}. The contributions of this paper are summarized as follows:
\vspace{-0.1cm}
\begin{itemize}
\item This paper reveals that existing SOTA OOD detection methods may suffer from catastrophic degradation in terms of OOD generalization. That is, their superior OOD detection ability is achieved by (secretly) sacrificing OOD generalization ability. We theoretically demystify the sensitive-robust dilemma in learning objectives as the main reason behind such a limitation.
\item In contrast to previous works that characterize OOD detection and generalization as conflictive learning tasks and thus implying an inevitable trade-off, we propose a novel learning framework termed Decoupled Uncertainty Learning (DUL) to successfully break through the limitation beyond a simple trade-off. Our DUL substantially harmonizes the conflict between OOD detection and OOD generalization, which achieves the best OOD detection performance without sacrificing the OOD generalization ability.
\item We conduct extensive experiments on standard benchmarks to validate our findings. Our DUL achieves dual-optimal OOD detection and OOD generalization performance. To our best knowledge, DUL is the first method that gains state-of-the-art OOD detection performance without sacrificing OOD generalization ability.
\end{itemize}

\section{Related works}

\textbf{OOD detection} aims to indicate whether the input arises from unknown classes that are not present in training data, which is essential for model trustworthiness. In the classification task, the majority of advanced OOD detection methods include MSP detectors which characterize samples with lower max softmax probability as OOD~\cite{hendrycks2016baseline, liang2017enhancing, lee2017training, hendrycks2018deep, hendrycks2019using}. EBM detectors identify high energy samples as OOD and frequently establish better performance than MSP detectors~\cite{liu2020energy,ming2022poem,chen2021atom,bai2023feed}, and various other types OOD detection methods such as distance-based detectors~\cite{lee2018simple}, non-parametric KNN-based detectors~\cite{sun2022out} which also show promises. According to the training paradigm, OOD detection methods can be split into auxiliary OOD-free and auxiliary OOD-required methods. Auxiliary OOD-free methods directly use the model pre-trained on ID data only for OOD detection. Another line of methods assumes that some OOD data is accessible during training and incorporates auxiliary outlier datasets (collected from websites or other datasets) for further enhancing OOD detection performance. By exposing the model to some semantic OOD during training, auxiliary OOD-required methods frequently outperform auxiliary OOD-free methods on commonly-used benchmarks~\cite{yang2022openood,zhang2023openood}.

\textbf{OOD generalization} expects the model to be robust under unforeseeable noise or corruption~\cite{arjovsky2019invariant,hu2013kullback,michel2021modeling,sinha2017certifying,han2022umix}. Basically, OOD generalization expects invariant and confident prediction on OOD data. Examples include classic domain adaption (DA) methods which encourage the model's behavior to be invariant across different distributions~\cite{arjovsky2019invariant,zhao2019learning,zhang2019bridging}. Besides, test-time adaption (TTA) directly encourages confident predictions on OOD data by minimizing predictive entropy~\cite{wang2020tent,zhang2022memo,niu2023towards}. However, as we will show later, confident prediction and invariance are seemingly conflictive to OOD detection purpose and further imply an unavoidable trade-off. The most related work to our paper is SCONE~\cite{bai2023feed}, which strikes to keep a balance between OOD detection and generalization performance. We argue that such a trade-off is not necessary and the conflict can be elegantly eliminated.

\textbf{Uncertainty estimation in Bayesian framework.} In the Bayesian framework, predictive uncertainty can be regarded as an indicator of whether the input sample is prone to be OOD. Since OOD samples are unseen during training and thus should be of higher uncertainty than ID. The overall predictive uncertainty of a classification model can be decomposed into three factors according to their source, including data (aleatoric) uncertainty (AU), distributional uncertainty (DU), and model (epistemic) uncertainty (EU)~\cite{kendall2017uncertainties,malinin2018predictive}. AU measures the natural complexity of the data (e.g., class overlap, label noise) and EU results from the difficulty of estimating the model parameters with finite training data. DU arises due to the mismatch between the distributions of test and training data. A line of classic measurement can be used to capture various types of uncertainty including entropy, mutual information, and differential entropy~\cite{malinin2018predictive}.

\section{Preliminaries \label{sec:preliminaries}}
We consider $K$-class classification task with classifier $f:\mathcal{X}\rightarrow\mathbb{R}^K$ parameterized by $\theta$, where $\mathcal{X}$ is the input space and $\mathcal{Y}=\{1,2,...,K\}$ denotes the target space. The model output $f_{\theta}(x)$ is considered as logits. The $k$-th element in logits is denoted as $f_{k}(x)$ indicates the confidence of predicting $x$ to class $k$. The predicted distribution $F(x)$ is obtained by normalizing $f(x)$ with the softmax function.
We first formalize all possible distributions that the model might encounter.

\begin{itemize}
\item In-distribution $P^{\rm ID}_{\mathcal{XY}}$ which denotes the distribution of labeled training data.

\item Covariate-shifted OOD $P^{\rm COV}_{\mathcal{XY}}$ which is relevant to OOD generalization. $P^{\rm COV}_{\mathcal{XY}}$ is of the same label space with ID. However, its marginal distribution $P^{\rm COV}_{\mathcal{X}}$ encounters shifts due to unexpected noise or corruption.

\item Semantic OOD $P^{\rm SEM}_{\mathcal{XY'}}$ is the distribution of data that do not belong to any known class. Its label space has no overlap with the known ID label space, i.e., $\mathcal{Y'}\cap\mathcal{Y}=\emptyset$. 
\end{itemize}

In the following paper, we omit the subscript for simplicity. The goal of OOD detection is to build a detector $G:\mathcal{X}\rightarrow [\rm IN, OUT]$ to decide whether an input $x$ is semantic OOD data or not through a thresholding function $G$ deduced from classifier $f$
\begin{equation}
    G_{\gamma}(x)=\left\{
\begin{aligned}
&{\rm IN}\ &g_{f}(x)\leq \gamma\\
&{\rm OUT}\ &g_{f}(x)>\gamma
\end{aligned}
\right.\ ,
\end{equation}
where $\gamma$ is the threshold. $g_f$ is an OOD scoring function deduced from $f$, which is expected to assign a higher value to OOD than ID. For example, in MSP detectors, $g_f(x)=-{\rm max}_k\ F(x)$ where $F(x)$ is the predicted softmax probability (negative max softmax probability). In EBM detectors, $g_f$ is realized by the energy function $E(x;f):=-\log\sum_{i=1}^K e^{f_k(x)}$ and the semantic input of OOD should be of high energy~\cite{liu2020energy}. Since it is difficult to foresee $P^{\rm SEM}$ one will encounter, a board line of OOD detection works~\cite{hendrycks2018deep,liu2020energy,malinin2018predictive,ming2022poem,chen2021atom,du2022vos,wang2024learning,choi2023balanced} regularize the model on some auxiliary OOD data $P^{\rm SEM}_{\rm train}$ during training (e.g., data from the web or other datasets), and expect the model can learn useful heuristic to handle unknown test-time OOD $P^{\rm SEM}_{\rm test}$. The learning objective is shown as follows
 \begin{equation}
 \label{eq:ood-target-1}
     \underset{\theta}{\rm min}\ \mathbb{
     E}_{(x,y)\sim P^{\rm ID}}[\mathcal{L}_{\rm CE}(f(x),y)]+\lambda \mathbb{E}_{\tilde{x}\sim P^{\rm SEM}_{\rm train}}[\mathcal{L}_{\rm reg}f(\tilde{x})],
 \end{equation}
where $\mathcal{L}_{\rm CE}$ is the standard cross entropy loss for the original classification task. $\mathcal{L}_{\rm reg}$ is the OOD detection regularization term depending on the detector used, which generally encourages a high uncertainty on $P^{\rm SEM}_{\rm train}$. For example, $\mathcal{L}_{\rm reg}$ is set to cross entropy between $F(x)$ and the uniform distribution for MSP detector~\cite{hendrycks2018deep}. In EBM detectors~\cite{liu2020energy}, $\mathcal{L}_{\rm reg}$ is realized as a margin ranking loss to explicitly encourage a large energy gap between ID and semantic OOD. In this paper, we are interested in this setting for the following reasons: 
\begin{itemize}
\item[1)]  In contrast to labeled data in supervised learning literature, auxiliary OOD data can be unlabeled and easy-to-collected in practice~\cite{ming2022poem}.

\item[2)]  Most SOTA methods involve auxiliary outliers~\cite{zhang2023openood,yang2022openood} for superior performance. 

\item[3)]  Even under some strict assumptions that $P^{\rm SEM}_{\rm train}$ is unavailable, recent works utilize GAN~\cite{lee2017training}, diffusion model~\cite{du2024dream} or sampling strategy~\cite{du2022vos} to generate ``virtual" outliers for training. 
\end{itemize}
Thus we believe this setting is promising and the cost of auxiliary outliers is minor given the importance of ensuring model trustworthiness. At test-time, the model is evaluated in terms of 
\begin{itemize}
\item ID accuracy (ID-Acc $\uparrow$) which measures the model's performance on $P^{\rm ID}$,
\item OOD accuracy (OOD-Acc $\uparrow$) measures the OOD generalization ability on $P^{\rm COV}$, 
\item False positive rate at 95\% true positive rate (FPR95$\downarrow$) := $\mathbb{E}_{x\sim P^{\rm SEM}_{\rm test}}(\mathbb{I}(G_{\gamma}(x)={\rm IN}))$ measures the OOD detection ability, where $\gamma$ is chosen when true positive rate (TPR) is $95\%$. $\mathbb{I}$ is the indicator function. In OOD detection, ID samples are considered as negative.
\end{itemize}
It is worth noting that in the standard OOD detection setting~\cite{ming2022poem,bai2023feed}, the test OOD data should not have any overlapped classes or samples with training-time auxiliary OOD data $P^{\rm SEM}_{\rm train}$. Let $\mathcal{Y}^{\rm SEM}_{\rm test}$ and $\mathcal{Y}^{\rm SEM}_{\rm train}$ be the label space of $P^{\rm SEM}_{\rm test}$ and $P^{\rm SEM}_{\rm train}$ respectively, we have $\mathcal{Y}^{\rm SEM}_{\rm test}\cap \mathcal{Y}^{\rm SEM}_{\rm train}=\emptyset$. Otherwise, OOD detection would be a trivial problem.

\section{Sensitive-robust Dilemma of Out-of-distribution Detection}\label{sec:theory}
In this section, we detail the limitation of current OOD detection methods: their OOD detection performance is achieved at the cost of generalization ability. This limitation implies the potential risk of SOTA OOD detection methods and underscores the urgent need for a better solution. Firstly, we re-examine representative OOD detection methods of six different types, including 1) baseline model MSP that is trained without any OOD detection regularization~\cite{hendrycks2016baseline}, 2) entropy-regularization (Entropy) that encourages high predictive entropy on OOD~\cite{hendrycks2018deep}, which is devised for MSP detectors, 3) energy-regularization for EBM detectors that enforces the output with high energy score for OOD input~\cite{liu2020energy}, 4) Bayesian uncertainty learning that encourages high overall uncertainty on OOD~\cite{malinin2018predictive}, 5) state-of-the-art OOD detection methods WOODS~\cite{katz2022training} and POEM~\cite{ming2022poem} 6) the most related SCONE~\cite{bai2023feed} that seeks for a trade-off between OOD detection and generalization performance.

\textbf{Limitation of current OOD detection methods.} In Fig.~\ref{fig:coverfigure} (b), we investigate current OOD detection methods in terms of OOD classification error and FPR95. The expected classifier should yield both low OOD classification error and FPR95. As it is observed, despite the superior OOD detection performance, all above methods significantly underperform the baseline MSP in terms of OOD generalization. By contrast, our method (DUL) successfully overcomes the limitation.

\textbf{Theoretical justification.} Toward understanding the limitation, we provide theoretical analysis for two types of most popular OOD detection methods, i.e., MSP and EBM detectors. Our analysis identifies the ``\textit{sensitive-robust}'' dilemma as the main reason behind such a limitation. The roadmap of our analysis is: (1) inspired by transfer learning theory, we first reveal that OOD detection regularization applied on semantic OOD may also affect the behavior of model on covariate-shifted OOD; (2) then we demonstrate why MSP detectors suffer from poor generalization by characterizing its generalization error bound; (3) we further identify that EBM methods~\cite{liu2020energy} suffer from a similar drawback when incorporating with gradient-based optimization. First of all, we recap the definition of disparity discrepancy in transfer learning theory~\cite{ganin2015unsupervised, zhang2019bridging}.
\begin{definition}[Disparity with Total Variation Distance] Given two hypotheses $f',f\in\mathcal{F}$ and distribution $P$, we define the Disparity with Total Variation Distance between them as
\begin{equation}
\label{eq:DD with TVD}
    {\rm disp}_P(f',f)=\mathbb{E}_P[TV(F_f||F_{f'})],
\end{equation}
where $F_{f'},F_{f}$ are the class distributions predicted by $f',f$ respectively. $TV(\cdot||\cdot)$ is the total variation distance, i.e., $TV(F_f||F_{f'})=\frac{1}{2}\sum_{k=1}^K||F_{f,k}-F_{f',k}||$. 
\end{definition}
\begin{definition}[Disparity Discrepancy with Total Variation Distance, DD with TVD] Given a hypothesis space $\mathcal{F}$ and two distributions $P,Q$, the Disparity Discrepancy with Total Variation Distance (DD with TVD) is defined as
\begin{equation}
\label{eq:DD}
    d_{\mathcal{F}}(P,Q):=\underset{f',f\in\mathcal{F}}{\rm sup}({\rm disp}_{P}(f',f)-{\rm disp}_{Q}(f',f)).
\end{equation}
\end{definition}
Disparity discrepancy (DD) measures the "distance" between two distributions $P,Q$ which considers the hypothesis space. DD is one of the most fundamental conceptions in transfer learning theory which constrains the behavior of hypothesis in $\mathcal{F}$ should not be dissimilar substantially on different distributions $P$ and $Q$. \footnote{Encompassed by \cite{mansour2008domain} as a special case, our definition is realized using TVD for theoretical convenience.} If the DD between semantic OOD and covariate-shifted OOD is limited, one can suppose that OOD detection regularization applied to semantic OOD samples will also influence the model's behavior on covariate-shifted OOD. Thus encouraging high uncertainty on semantic OOD may also result in highly uncertain prediction on covariate-shifted OOD, which is potentially harmful to generalization ability. We first formalize this intuition for MSP detectors.
\begin{theorem}[Sensitive-robust dilemma]
\label{Theorem 1} Let $\mathcal{P^{\rm COV}}$, $P^{\rm SEM}_{\rm test}$ be the covariate-shifted OOD and semantic OOD distribution. ${\rm GError}_{P^{\rm COV}}(f)$ denotes standard cross entropy loss taking expectation on $P^{\rm COV}$, i.e., generalization error. Then we have
\begin{equation}
\underbrace{{\rm GError}_{P^{\rm COV}}(f)}_{\rm OOD\ generalization\ error} \geq \ C-\sqrt{2}\ \mathbb{E}_{P^{\rm SEM}_{\rm test}}[\underbrace{\mathcal{L}_{\rm reg}(f)}_{\rm OOD\ detection\ loss}-\log K]^{\frac{1}{2}}-2d_{\mathcal{F}}(\mathcal{P^{\rm COV}},\mathcal{P^{\rm SEM}_{\rm test}}),
\end{equation}
where $\mathcal{L}_{\rm reg}$ is the OOD detection loss devised for MSP detectors defined in~\cite{hendrycks2018deep}, i.e., cross-entropy between predicted distribution $F(x)$ and uniform distribution. $d_{\mathcal{F}}(P^{\rm COV},P^{\rm SEM}_{\rm test})$ is DD with TVD that measures the dissimilarity of covariate-shifted OOD and semantic OOD. $C$ is some constant depending on hypothesis space $\mathcal{F}$, $P^{\rm COV}$ and $P^{\rm SEM}_{\rm test}$.
\end{theorem}
The proof is deferred in Appendix~\ref{appendix:proof}. Theorem~\ref{Theorem 1} demonstrates that for MSP detectors, the OOD detection objective conflicts with OOD generalization. The model's generalization error lower bound is negatively correlated with OOD detection loss that the model tries to minimize. Thus given a limited $d_{\mathcal{F}}$, pursuing low OOD detection loss on $P^{\rm SEM}_{\rm test}$ will also inevitably result in highly uncertain prediction on $P^{\rm COV}$. It is worth noting that such an interpretative theorem is applicable for all MSP-based OOD detectors no matter whether the model involves $P^{\rm SEM}_{\rm train}$ during training or not. Since the inherent motivation of OOD detection methods lies in minimizing the OOD detection loss in $P^{\rm SEM}_{\rm test}$, regardless of the training strategies used.

\textbf{Why a limited $d_{\mathcal{F}}(P^{\rm COV},P^{\rm SEM}_{\rm test})$ is practical?} In Theorem~\ref{Theorem 1}, $d_{\mathcal{F}}(P^{\rm COV},P^{\rm SEM}_{\rm test})$ measures the dissimilarity between $P^{\rm COV}$ and $P^{\rm SEM}_{\rm test}$. It seems that this lower bound will be very small and trivial when $d_{\mathcal{F}}(P^{\rm COV},P^{\rm SEM}_{\rm test})$ is large enough. However, since the semantic OOD samples can be any samples that do not belong to ID classes, one can suppose that semantic OOD samples are extremely diverse and some are of high similarity with ID and covariate-shifted OOD~\cite{yang2023imagenet}. Detecting these ``ID-like'' OOD samples is inherently the core challenge of OOD detection~\cite{ming2022poem,chen2021atom,bai2023id}. Thus, it is reasonable to assume a limited $d_{\mathcal{F}}(P^{\rm COV},P^{\rm SEM}_{\rm test})$. We provide more discussions in the Appendix~\ref{appendix:discussion}.

As presented before, the key limitation of MSP detectors is that they enforce high-entropy prediction on semantic OOD. We proceed to reveal that EBM detectors suffer from similar issues due to the natural property of gradient-based optimization. For EBM detectors, $\mathcal{L}_{\rm reg}$ is defined by~\cite{liu2020energy} is
\begin{equation}
    \mathcal{L}_{\rm reg}=\mathbb{E}_{\tilde{x}\sim{P}^{\rm SEM}_{\rm train}}[{\rm max}(m_{\rm OUT}-E(\tilde{x}),0)]^2+\mathbb{E}_{x\sim P^{\rm ID}}[{\rm max}(0,E(x)-m_{\rm IN})]^2,
\end{equation}
which constrains the energy score $E(x;f):=-log\sum_{k=1}^K e^{f_k(x)}$ of ID sample $x$ to be lower than that of OOD sample $\tilde{x}$. $m_{\rm IN}, m_{\rm OUT}$ are manually selected margins. Although such regularization does not indicate high entropy prediction at first glance, unfortunately, we demonstrate that EBM detectors also tend to uncertain prediction when equipped with gradient-based optimization. Here we focus on Gradient Descent (GD) as a showcase. In each training epoch $t$, model $f_{\theta}$ is updated with GD as $\theta_{t+1}=\theta_{t}-\eta\nabla_{\theta}\mathcal{L}_{\rm reg}$, where $\eta$ is the learning rate. For any sample $\tilde{x}$ drawn from $\mathcal{P^{\rm SEM}}$, the gradient $\nabla_{\theta}\mathcal{L}_{\rm reg}$ can be written as
\begin{equation}
\begin{split}
2\sum_{k=1}^K \nabla f_k(\tilde{x})\underbrace{[e^{f_k(\tilde{x})}(\sum_{k=1}^K e^{f_k(\tilde{x})})^{-1}]}_{{\rm higher\ for\ larger}\ f_k(\tilde{x})}(m_{\rm OUT}-E(\tilde{x})),
\end{split}
\end{equation}
where $m_{\rm OUT}-E(\tilde{x})>0$ (otherwise the gradient is zero) and $f_k(\tilde{x})$ is the predicted logits on the $k$-th class. For sample $\tilde{x}$, when class $k$ has larger predicted logit, it contributes more to the overall gradient $\nabla_{\theta}\mathcal{L}_{\rm reg}$ and thus could obtain more optimization efforts during backpropagation. Eventually, one can infer that when an EBM detector is about to converge, it tends to high-entropy prediction on $P^{\rm SEM}_{\rm train}$ accordingly. Incorporating this into the established Theorem~\ref{Theorem 1}, this is likely to harm the generalization ability of $P^{\rm COV}$. Empirical evidence can support this supposition (see Table~\ref{tab:entropy} in Appendix~\ref{appendix:full results}). Therefore both MSP and EBM detectors face the ``\textit{sensitive-robust}'' dilemma.

\section{Decoupled Uncertainty Learning}\label{sec:method}
We demonstrate how to handle the dilemma between OOD detection and generalization by decoupled uncertainty learning in the Bayesian framework. Unlike the most related work~\cite{bai2023feed} which aims to seek a good trade-off, our method successfully gets out of the aforementioned sensitive-robust dilemma.

\textbf{Uncertainty Estimation in Bayesian Framework.} We first revisit the theoretical properties of different types of uncertainty in a Bayesian framework. Non-Bayesian classifiers consider the model's output $f(x)$ as logits, which is then normalized with softmax to directly model predictive categoricals $p(\hat{y}|x)$. While in Bayesian framework~\cite{malinin2018predictive}, $f(x)$ is considered as parameters of a Dirichlet distribution $p(\mu|x)$ firstly, which is used to model the prior of predictive categoricals $p(\hat{y}|x)$ by
\begin{equation}
    p(\mu|x)={\rm Dir}(\mu|\boldsymbol{\alpha})=\frac{\Gamma(\alpha_0)}{\prod_{k=1}^K\Gamma(\alpha_k)}\prod_{k=1}^K \mu^{\alpha_k-1}_k, \ \boldsymbol{\alpha}=f(x),
\end{equation}
where ${\rm Dir}(\mu|\boldsymbol{\alpha})$ is Dirichlet distribution and $\boldsymbol{\alpha}$ is the concentration parameters of Dirichlet. The sum of all $\alpha_k\in\boldsymbol{\alpha}$ (noted as $\alpha_0$) is so called the strength of the Dirichlet distribution, i.e., $\alpha_k>0, \alpha_0=\sum_{k}\alpha_k$. After obtaining prior $p(\mu|x,\theta)$, the final predicted posterior $p(\hat{y}|x)$ over class labels is given by calculating the mean of the Dirichlet prior
\begin{equation}
    \begin{split}
        \underbrace{p(\hat{y}|x,\theta)}_{\rm overall\ uncertainty}\hspace{-4mm}=\int \hspace{-4mm}\overbrace{p(\hat{y}|\mu)}^{\rm data\ uncertainty}\hspace{-12mm}\underbrace{p(\mu|x,\theta)}_{\rm distributional\ uncertainty}\hspace{-6mm}d\mu.
    \end{split}
\end{equation}
From a Bayesian perspective, given deterministic parameters $\theta$, the overall uncertainty of final prediction $p(\hat{y}|x)$ can be decomposed into two factors, including data (aleatoric) uncertainty (AU) and distributional uncertainty (DU). DU lies in $p(\mu|x,\theta)$ which is defined as uncertainty due to the mismatch between the distributions of test and train data. AU is described by $p(\hat{y}|\mu)$ which captures the natural complexity of the data (e.g., class-overlap)~\cite{malinin2018predictive}. By definition, OOD detection is primarily associated with DU which is only a part of the overall uncertainty. While generalization is related to the overall uncertainty of $p(\hat{y}|x)$ as we mentioned in related works (both AU and DU). One essential property of DU is that it can be high even if the expected categorical $p(\hat{y}|x,\theta)$ expresses low overall uncertainty. Such a property is well suited to achieve OOD detection and generalization jointly since high DU no longer necessarily indicates high overall uncertainty.

\textbf{Decoupled Uncertainty Learning.} While the aforementioned Bayesian framework enjoys theoretical potentiality, its learning object~\cite{malinin2018predictive} lacks consideration of OOD generalization. Similar to other OOD detection methods, it also directly enforces high overall uncertainty on OOD
\begin{equation}
\label{eq:DPN}
    \min_{\theta} \mathbb{E}_{\mathcal{P}^{\rm ID}}{\rm KL}(p(y|x))||p(\hat{y}|x))+\mathbb{E}_{\mathcal{P}^{\rm SEM}_{\rm train}}{\rm KL}(p(\mu|\tilde{x}))||{\rm Dir}(\mu|\alpha=\mathbf{1})),
\end{equation}
where $p(y|x),p(\hat{y}|x)$ are the ground-truth distribution and predicted distribution on ID. The model's prediction on OOD is enforced to be close to a rather flat Dirichlet distribution. It is worth noting that ${\rm Dir}(\mu|\alpha=\mathbf{1})$ means all classes are equiprobable, and the entropy of the final prediction is maximized. As shown in Fig.~\ref{fig:coverfigure} (b), the vanilla Bayesian method~\cite{malinin2018predictive} also suffers from degraded OOD generalization performance. To this end, we propose \textbf{D}ecoupled \textbf{U}ncertainty \textbf{L}earning (DUL), a novel OOD detection regularization method that explicitly encourages high DU on OOD samples without affecting the overall uncertainty. Similarly to previous OOD detection methods~\cite{liu2020energy}, our DUL is also devised in a finetune manner for effectiveness. Given a classifier $f_{\theta_0}$ well pre-trained on $P^{\rm ID}$, the goal of DUL lies in enhancing its OOD detection performance without sacrificing any generalization ability. Specifically, we finetune the model by encouraging higher DU but non-increased overall uncertainty on $P^{\rm SEM}_{\rm train}$ . The learning objective of DUL is
\begin{equation}
\begin{split}
\label{eq:object-theory}
    \underset{\theta}{\rm min}\ &\underbrace{\mathbb{E}_{(x,y)\sim P^{\rm ID}}[\mathcal{L}_{\rm CE}(f(x),y)]}_{\rm ID\ classification}+\lambda\underbrace{\mathbb{E}_{\tilde{x}\sim P^{\rm SEM}_{\rm train}}||{\rm max}(0,(h_0+m_{\rm OUT})-h)||_{\tau}}_{\rm high\ distributional\ uncertainty\ (detection)}
  \\   &\ {\rm s.t.}\ \underbrace{H(p(\hat{y}|\tilde{x}))= H(p_0(\hat{y}|\tilde{x}))}_{\rm non-increased\ overall\ uncertainty\ (generalization)}\ \forall\ \tilde{x}\sim\ {P^{\rm SEM}_{\rm train}},
\end{split}
\end{equation}
where $H(\cdot)$ is the entropy. $p(\hat{y}|\tilde{x})$ and $p_0(\hat{y}|\tilde{x})$ are the predicted distribution on semantic OOD data $\tilde{x}$ after and before finetuning. The first term is the original ID classification loss. The second term is OOD detection loss, which encourages high DU on outlier $\tilde{x}$. $m_{\rm OUT}$ and $\tau>0$ are hyperparameters. $h_0,h$ are DU on $\tilde{x}$ before and after finetuning. Here we measure DU with the differential entropy ($h[p(\mu|\tilde{x},\theta)]=-\int_{S^{K-1}} p(\mu|\tilde{x}){\rm ln}(p(\mu|\tilde{x}))d\mu$). We refer interested readers to the Appendix~\ref{appendix:derivation} for mathematical details. The third term constraining on $H(p(\hat{y}|\tilde{x}))$ avoids increment of overall uncertainty during finetuning and thus the generalization ability can be retained. Considering the difficulty of constrained optimization, we convert Eq.~\ref{eq:object-theory} into an unconstrained form and get our final minimizing objective
\begin{equation}
\label{eq:object-practice}
    \underbrace{\mathbb{E}_{P^{\rm ID}}[\mathcal{L}_{\rm CE}(f(x),y)]}_{\rm ID\ classification}+\mathbb{E}_{P^{\rm SEM}_{\rm train}}\{\lambda\underbrace{||{\rm max}(0,(h_0+m_{\rm OUT})-h)||_{\tau}}_{\rm high\ distributional\ uncertainty}
     +\underbrace{\gamma{\rm KL}(p(\hat{y}|\tilde{x})||p_0(\hat{y}|\tilde{x}))}_{\rm unchanged\ overall\ uncertainty}\},
\end{equation}
where $\gamma$ is hyperparameter. In contrast to previous Bayesian method~\cite{malinin2018predictive}, DUL only encourages high DU rather than overall uncertainty on OOD and explicitly discourages high entropy in the final prediction. The implementation details are in Appendix~\ref{appendix:derivation}.

\section{Experiment}
We conduct experiments to validate our analysis and the superiority of DUL. The questions to be verified are Q1 Motivation. To what extent does OOD detection conflict with OOD generalization in previous methods? Q2 Effectiveness. Does DUL achieve better OOD detection and generalization performance compared to its counterparts? Q3 Interpretability. Does the proposed method well decouple uncertainty as expected?

\begin{table}[!ht]
\begin{center}

\vskip -0.4in
\caption{OOD detection and generalization performance comparison. Substantial ($\geq 0.5$) \textcolor{mycolor4}{improvement} and \textcolor{mycolor3}{degradation} compared to the baseline MSP~\cite{hendrycks2016baseline} (training without any OOD detection regularization) are highlighted in blue or red respectively. The \textbf{best} and \underline{second-best} results are in bold or underlined. DUL is the only method that achieves SOTA OOD detection performance (mostly the best or second best) without sacrificing generalization i.e., the value of the entire row is either blue or black. Full results with standard deviation and diverse types of corruption are in Appendix~\ref{appendix:full results}.}
\vskip -0.1in
\label{tab:main}
\center
\resizebox{1.0\textwidth}{!}{
\setlength{\tabcolsep}{3.9mm}
\begin{tabular}{cc|cc|ccc}
\toprule  \multirow{2}{*}{$\mathcal{P}^{\rm ID} / \mathcal{P}^{\rm SEM}_{\rm train}$}   &\multirow{2}{*}{Method}   &  \multicolumn{2}{c|}{Model generalization}                                                    &\multicolumn{3}{c}{OOD detection}\\
   &  &\text{ID-Acc $\uparrow^{\phantom{\pm01}}$} & \text{OOD-Acc $\uparrow^{\phantom{\pm01}}$} & \text{FPR $\downarrow^{\phantom{\pm0.0}}$}& \text{AUROC $\uparrow^{\phantom{\pm0.0}}$}  &\text{AUPR $\uparrow^{\phantom{\pm0.0}}$}  \\ \midrule
         
      \multirow{4}{*}{\shortstack{CIFAR-10 / \\ None}} &MSP        &  $96.11^{\phantom{\pm0.01}}$   &  $87.35^{\phantom{\pm0.01}}$                                                  &$41.96^{\phantom{\pm0.01}}$  & $89.28^{\phantom{\pm0.01}}$  & $68.00^{\phantom{\pm0.01}}$                          \\  &EBM (pretrain)      &  $96.11^{\phantom{\pm0.01}}$&  $87.35^{\phantom{\pm0.01}}$                                                  & $32.45^{\phantom{\pm0.01}}$ & $89.34^{\phantom{\pm0.01}}$   &$75.22^{\phantom{\pm0.01}}$                              \\ &Maxlogits      &  $96.11^{\phantom{\pm0.01}}$&  $87.35^{\phantom{\pm0.01}}$                                                  &$32.90^{\phantom{\pm0.01}}$ & $89.26^{\phantom{\pm0.01}}$   & $74.47^{\phantom{\pm0.01}}$                            \\ &Mahalanobis      &  $96.11^{\phantom{\pm0.01}}$&  $87.35^{\phantom{\pm0.01}}$                                                  & $32.53^{\phantom{\pm0.01}}$ & $93.93^{\phantom{\pm0.01}}$   & $74.96^{\phantom{\pm0.01}}$                                                            \\ \midrule \multirow{7}{*} {\shortstack{CIFAR-10 / \\ ImageNet-RC}} 
    
     &Entropy      &   $\underline{96.04}^{\phantom{\pm0.01}}$    & \textcolor{mycolor3}{$72.57^{\phantom{\pm0.01}}$}                                                    &\textcolor{mycolor4}{$6.63^{\phantom{\pm0.01}}$}  &\textcolor{mycolor4}{$\underline{98.72}^{\phantom{\pm0.01}}$}   &\textcolor{mycolor4}{$94.00^{\phantom{\pm0.01}}$}                            \\  
    
     &EBM (finetune)      &   $\textbf{96.10}^{\phantom{\pm0.01}}$     &\textcolor{mycolor3}{$79.03^{\phantom{\pm0.01}}$}                                                    &\textcolor{mycolor4}{$\underline{3.61}^{\phantom{\pm0.01}}$}  &\textcolor{mycolor4}{$98.39^{\phantom{\pm0.01}}$}   &\textcolor{mycolor4}{$94.88^{\phantom{\pm0.01}}$}                              \\  
    
     &POEM      &   \textcolor{mycolor3}{$94.32^{\phantom{\pm0.01}}$}     &\textcolor{mycolor3}{$78.89^{\phantom{\pm0.01}}$}                                                    &\textcolor{mycolor4}{$\textbf{3.32}^{\phantom{\pm0.01}}$}  &\textcolor{mycolor4}{$\textbf{98.99}^{\phantom{\pm0.01}}$}   & \textcolor{mycolor4}{$\textbf{99.38}^{\phantom{\pm0.01}}$}                 \\ &DPN      &   $95.69^{\phantom{\pm0.01}}$     &\textcolor{mycolor3}{${85.52}^{\phantom{\pm0.01}}$}                                                    &\textcolor{mycolor4}{${4.28}^{\phantom{\pm0.01}}$}  &\textcolor{mycolor4}{${98.53}^{\phantom{\pm0.01}}$}   & \textcolor{mycolor4}{${94.93}^{\phantom{\pm0.01}}$}                             \\ &WOODS      &   $96.01^{\phantom{\pm0.01}}$     &\textcolor{mycolor3}{${80.14}^{\phantom{\pm0.01}}$}                                                    &\textcolor{mycolor4}{${7.12}^{\phantom{\pm0.01}}$}  &\textcolor{mycolor4}{${98.45}^{\phantom{\pm0.01}}$}   & \textcolor{mycolor4}{${92.46}^{\phantom{\pm0.01}}$}                             \\        &SCONE      &  $95.96^{\phantom{\pm0.01}}$ &\textcolor{mycolor3}{$78.80^{\phantom{\pm0.01}}$}                                                    &\textcolor{mycolor4}{$7.02^{\phantom{\pm0.01}}$}  &\textcolor{mycolor4}{$98.45^{\phantom{\pm0.01}}$}   &\textcolor{mycolor4}{$92.46^{\phantom{\pm0.01}}$}                             \\
      &DUL (ours)      & \cellcolor{gray!20}$96.02^{\pm0.00}$
 &\cellcolor{gray!20}\textcolor{mycolor4}{$\textbf{88.01}^{\pm0.29}$}                                                    &\cellcolor{gray!20}\textcolor{mycolor4}{$5.89^{\pm0.12}$}  &\cellcolor{gray!20}\textcolor{mycolor4}{$98.47^{\pm0.02}$}   &\cellcolor{gray!20}\textcolor{mycolor4}{$92.44^{\pm1.29}$} 
     \\  
     &DUL\textsuperscript{\dag} (ours)      &  \cellcolor{gray!20}$\underline{96.04}^{\pm0.00}$  &\cellcolor{gray!20}{$\underline{87.53}^{\pm0.49}$}                                                    &\cellcolor{gray!20}\textcolor{mycolor4}{${5.99}^{\pm0.06}$}  &\cellcolor{gray!20}\textcolor{mycolor4}{$98.28^{\pm0.01}$}   & \cellcolor{gray!20}\textcolor{mycolor4}{$\underline{98.40}^{\pm0.13}$} 
     \\ \midrule
\multirow{7}{*} {\shortstack{CIFAR-10 / \\ TIN-597}}  
     &Entropy       &   $\underline{95.94}^{\phantom{\pm0.01}}$    &\textcolor{mycolor3}{$80.51^{\phantom{\pm0.01}}$}                                                    &\textcolor{mycolor4}{$11.60^{\phantom{\pm0.01}}$}  &\textcolor{mycolor4}{$97.93^{\phantom{\pm0.01}}$}   & \textcolor{mycolor4}{${92.16}^{\phantom{\pm0.01}}$}                             \\  
    
     &EBM (finetune)       &   \textcolor{mycolor3}{$95.38^{\phantom{\pm0.01}}$}     &\textcolor{mycolor3}{$83.67^{\phantom{\pm0.01}}$}                                                    &\textcolor{mycolor4}{$19.36^{\phantom{\pm0.01}}$}  &\textcolor{mycolor3}{$87.51^{\phantom{\pm0.01}}$}   &\textcolor{mycolor4}{$83.63^{\phantom{\pm0.01}}$}                              \\  
    
     &POEM       &   \textcolor{mycolor3}{$95.44^{\phantom{\pm0.01}}$}     &\textcolor{mycolor3}{$83.17^{\phantom{\pm0.01}}$}                                                    &\textcolor{mycolor4}{$24.34^{\phantom{\pm0.01}}$}  &\textcolor{mycolor3}{$86.83^{\phantom{\pm0.01}}$}   &\textcolor{mycolor4}{$\underline{94.25}^{\phantom{\pm0.01}}$}                 \\  &DPN      &   \textcolor{mycolor3}{$94.39^{\phantom{\pm0.01}}$} &\textcolor{mycolor3}{$79.23^{\phantom{\pm0.01}}$}                                                    &\textcolor{mycolor4}{$17.27^{\phantom{\pm0.01}}$}  &\textcolor{mycolor4}{$94.92^{\phantom{\pm0.01}}$}   &\textcolor{mycolor4}{$87.67^{\phantom{\pm0.01}}$}
                          \\       &WOODS      &  \textcolor{mycolor3}{$95.57^{\phantom{\pm0.01}}$} &\textcolor{mycolor3}{${84.68}^{\phantom{\pm0.01}}$}                                                    &\textcolor{mycolor4}{$\underline{7.58}^{\phantom{\pm0.01}}$}  &\textcolor{mycolor4}{$\textbf{98.29}^{\phantom{\pm0.01}}$}   & \textcolor{mycolor4}{${93.08}^{\phantom{\pm0.01}}$}                           \\       &SCONE      &  \textcolor{mycolor3}{$95.19^{\phantom{\pm0.01}}$} &\textcolor{mycolor3}{${84.68}^{\phantom{\pm0.01}}$}                                                    &\textcolor{mycolor4}{${8.02}^{\phantom{\pm0.01}}$}  &\textcolor{mycolor4}{$\underline{98.21}^{\phantom{\pm0.01}}$}   & \textcolor{mycolor4}{${93.08}^{\phantom{\pm0.01}}$}                           \\  
     &DUL (ours)      &  \cellcolor{gray!20}$\textbf{96.06}^{\pm0.01}$  &\cellcolor{gray!20}\textcolor{mycolor4}{$\underline{87.93}^{\pm0.39}$}                                                    &\cellcolor{gray!20}\textcolor{mycolor4}{$\textbf{6.87}^{\pm0.67}$}  &\cellcolor{gray!20}\textcolor{mycolor4}{$\underline{98.21}^{\pm0.01}$}   & \cellcolor{gray!20}\textcolor{mycolor4}{$91.29^{\pm1.39}$}                             \\  
     &DUL\textsuperscript{\dag} (ours)      &  \cellcolor{gray!20}$\underline{95.94}^{\pm0.01}$  &\cellcolor{gray!20}\textcolor{mycolor4}{$\textbf{88.10}^{\pm0.07}$}                                                    &\cellcolor{gray!20}\textcolor{mycolor4}{${10.34}^{\pm0.11}$}  &\cellcolor{gray!20}\textcolor{mycolor4}{${97.67}^{\pm0.01}$}   & \cellcolor{gray!20}\textcolor{mycolor4}{$\textbf{98.59}^{\pm0.06}$}                             \\ \midrule \midrule \multirow{4}{*} {\shortstack{CIFAR-100 / \\ None}}
 & MSP    &   $80.99^{\phantom{\pm0.01}}$   &   $55.95^{\phantom{\pm0.01}}$                                                 & $74.63^{\phantom{\pm0.01}}$ & $80.19^{\phantom{\pm0.01}}$  & $42.59^{\phantom{\pm0.01}}$                          \\  &EBM (pretrain)      &  $80.99^{\phantom{\pm0.01}}$  &$55.95^{\phantom{\pm0.01}}$                                                    &$67.42^{\phantom{\pm0.01}}$  &$82.67^{\phantom{\pm0.01}}$   &$49.35^{\phantom{\pm0.01}}$                              \\ 
    
     &Maxlogits      &  $80.99^{\phantom{\pm0.01}}$ &$55.95^{\phantom{\pm0.01}}$                                                    &$69.32^{\phantom{\pm0.01}}$  &$82.30^{\phantom{\pm0.01}}$   & $47.60^{\phantom{\pm0.01}}$                             \\  &Mahalanobis      &  $80.99^{\phantom{\pm0.01}}$ &$55.95^{\phantom{\pm0.01}}$                                                    &$61.51^{\phantom{\pm0.01}}$  &$85.97^{\phantom{\pm0.01}}$   & $56.10^{\phantom{\pm0.01}}$                              \\ \midrule \multirow{7}{*} {\shortstack{CIFAR-100 / \\ ImageNet-RC}}        &Entropy      &  \textcolor{mycolor3}{$80.21^{\phantom{\pm0.01}}$} &\textcolor{mycolor3}{$45.48^{\phantom{\pm0.01}}$}                                                    &\textcolor{mycolor4}{$22.29^{\phantom{\pm0.01}}$}  &\textcolor{mycolor4}{$95.33^{\phantom{\pm0.01}}$}   &\textcolor{mycolor4}{$82.34^{\phantom{\pm0.01}}$}                              \\  
    
     &EBM (finetune)      &   $80.53^{\phantom{\pm0.01}}$  &\textcolor{mycolor3}{$48.14^{\phantom{\pm0.01}}$}                                                    &\textcolor{mycolor4}{$13.47^{\phantom{\pm0.01}}$}  &\textcolor{mycolor4}{$\underline{96.78}^{\phantom{\pm0.01}}$}   &\textcolor{mycolor4}{${87.84}^{\phantom{\pm0.01}}$}                              \\  
    
     &POEM      &   \textcolor{mycolor3}{$78.15^{\phantom{\pm0.01}}$} &\textcolor{mycolor3}{$42.18^{\phantom{\pm0.01}}$}                                                    &\textcolor{mycolor4}{$\textbf{9.89}^{\phantom{\pm0.01}}$}  &\textcolor{mycolor4}{$\textbf{97.79}^{\phantom{\pm0.01}}$}   &\textcolor{mycolor4}{$\textbf{98.40}^{\phantom{\pm0.01}}$}                          \\  &DPN      &   \textcolor{mycolor3}{$78.90^{\phantom{\pm0.01}}$} &\textcolor{mycolor3}{$50.14^{\phantom{\pm0.01}}$}                              &\textcolor{mycolor4}{$18.36^{\phantom{\pm0.01}}$}  &\textcolor{mycolor4}{${95.42}^{\phantom{\pm0.01}}$}   &\textcolor{mycolor4}{$74.45^{\phantom{\pm0.01}}$}\\  &WOODS      &   {$80.69^{\phantom{\pm0.01}}$} &\textcolor{mycolor3}{$54.38^{\phantom{\pm0.01}}$}                              &\textcolor{mycolor4}{$38.15^{\phantom{\pm0.01}}$}  &\textcolor{mycolor4}{${92.01}^{\phantom{\pm0.01}}$}   &\textcolor{mycolor4}{$71.79^{\phantom{\pm0.01}}$}
                          \\      &SCONE      &  ${80.80}^{\phantom{\pm0.01}}$ &\textcolor{mycolor4}{$\textbf{56.73}^{\phantom{\pm0.01}}$}                                                    &\textcolor{mycolor4}{$47.60^{\phantom{\pm0.01}}$}  &\textcolor{mycolor4}{$89.61^{\phantom{\pm0.01}}$}   &\textcolor{mycolor4}{$65.29^{\phantom{\pm0.01}}$}                              \\ 
    
     &DUL (ours)      &  \cellcolor{gray!20}$\textbf{81.30}^{\pm0.04}$ &\cellcolor{gray!20}$\underline{56.27}^{\pm3.29}$                                                    &\cellcolor{gray!20}\textcolor{mycolor4}{${12.49}^{\pm0.05}$}  &\cellcolor{gray!20}\textcolor{mycolor4}{$95.24^{\pm0.01}$}   &\cellcolor{gray!20}\textcolor{mycolor4}{$86.72^{\pm0.58}$} 
      \\  
     &DUL\textsuperscript{\dag} (ours)      &  \cellcolor{gray!20}$\underline{81.23}^{\pm0.05}$  &\cellcolor{gray!20}{${55.41}^{\pm0.54}$}                                                    &\cellcolor{gray!20}\textcolor{mycolor4}{$\underline{11.12}^{\pm0.62}$}  &\cellcolor{gray!20}\textcolor{mycolor4}{$95.46^{\pm0.36}$}   & \cellcolor{gray!20}\textcolor{mycolor4}{$\underline{96.49}^{\pm0.13}$}\\  \midrule
\multirow{7}{*}{\shortstack{CIFAR-100 / \\ TIN-597}}        &Entropy      &  \textcolor{mycolor3}{${80.15}^{\phantom{\pm0.01}}$} &\textcolor{mycolor3}{$46.25^{\phantom{\pm0.01}}$}                                                    &\textcolor{mycolor4}{$26.88^{\phantom{\pm0.01}}$}  &\textcolor{mycolor4}{${93.50}^{\phantom{\pm0.01}}$}   &\textcolor{mycolor4}{$79.81^{\phantom{\pm0.01}}$}                             \\  
    
    &EBM (finetune)     &   \textcolor{mycolor3}{$79.94^{\phantom{\pm0.01}}$}  &\textcolor{mycolor3}{$50.00^{\phantom{\pm0.01}}$}                                                    &\textcolor{mycolor4}{$26.87^{\phantom{\pm0.01}}$}  &\textcolor{mycolor4}{$91.68^{\phantom{\pm0.01}}$}   & \textcolor{mycolor4}{$80.08^{\phantom{\pm0.01}}$}                             \\  &POEM      &   \textcolor{mycolor3}{$78.68^{\phantom{\pm0.01}}$} &\textcolor{mycolor3}{$52.53^{\phantom{\pm0.01}}$}                                                    &\textcolor{mycolor4}{$32.71^{\phantom{\pm0.01}}$}  &\textcolor{mycolor4}{$91.30^{\phantom{\pm0.01}}$}   &\textcolor{mycolor4}{$\underline{94.65}^{\phantom{\pm0.01}}$}                          \\  &DPN      &   \textcolor{mycolor3}{$78.44^{\phantom{\pm0.01}}$} &\textcolor{mycolor3}{$47.67^{\phantom{\pm0.01}}$}                              &\textcolor{mycolor4}{${24.99}^{\phantom{\pm0.01}}$}  &\textcolor{mycolor4}{$\underline{93.55}^{\phantom{\pm0.01}}$}   &\textcolor{mycolor4}{${81.63}^{\phantom{\pm0.01}}$}
                          \\      &WOODS      & \textcolor{mycolor3} {$79.26^{\phantom{\pm0.01}}$} &\textcolor{mycolor3}{${53.13}^{\phantom{\pm0.01}}$}                                                    &\textcolor{mycolor4}{$36.71^{\phantom{\pm0.01}}$}  &\textcolor{mycolor4}{$92.15^{\phantom{\pm0.01}}$}   &\textcolor{mycolor4}{$73.42^{\phantom{\pm0.01}}$}                              \\  &SCONE      & \textcolor{mycolor3} {$79.53^{\phantom{\pm0.01}}$} &\textcolor{mycolor3}{${52.70}^{\phantom{\pm0.01}}$}                                                    &\textcolor{mycolor4}{$35.60^{\phantom{\pm0.01}}$}  &\textcolor{mycolor4}{$92.47^{\phantom{\pm0.01}}$}   &\textcolor{mycolor4}{$73.58^{\phantom{\pm0.01}}$}                              \\ 
    
      &DUL (ours)      &  \cellcolor{gray!20}$\textbf{80.85}^{\pm0.06}$ &\cellcolor{gray!20}$\underline{56.19}^{\pm2.33}$                                                    &\cellcolor{gray!20}\textcolor{mycolor4}{$\underline{23.32}^{\pm1.22}$}  &\cellcolor{gray!20}\textcolor{mycolor4}{$\textbf{94.48}^{\pm0.12}$}   &\cellcolor{gray!20}\textcolor{mycolor4}{$80.82^{\pm2.63}$}  \
      \\  
     &DUL\textsuperscript{\dag} (ours)      &  \cellcolor{gray!20}$\underline{80.50}^{\pm0.06}$  &\cellcolor{gray!20}{$\textbf{56.22}^{\pm1.66}$}                                                    &\cellcolor{gray!20}\textcolor{mycolor4}{$\textbf{22.75}^{\pm0.78}$}  &\cellcolor{gray!20}\textcolor{mycolor4}{$90.88^{\pm0.08}$}   & \cellcolor{gray!20}\textcolor{mycolor4}{$\textbf{96.33}^{\phantom{\pm0.01}}$} \\
\midrule \midrule
\multirow{3}{*} {\shortstack{ImageNet-200 / \\ \\None}}         &MSP        &  $\underline{85.15}^{\phantom{\pm0.01}}$   &  $74.84^{\phantom{\pm0.01}}$                                                  &$58.23^{\phantom{\pm0.01}}$  & $86.98^{\phantom{\pm0.01}}$  & $82.27^{\phantom{\pm0.01}}$                          \\  &EBM (pretrain)      &  $\underline{85.15}^{\phantom{\pm0.01}}$&  $74.84^{\phantom{\pm0.01}}$                                                  & $51.94^{\phantom{\pm0.01}}$ & $88.18^{\phantom{\pm0.01}}$   &$84.75^{\phantom{\pm0.01}}$                              \\ &Maxlogits      &  $\underline{85.15}^{\phantom{\pm0.01}}$&  $74.84^{\phantom{\pm0.01}}$                                                  & $\underline{51.62}^{\phantom{\pm0.01}}$ & $88.30^{\phantom{\pm0.01}}$   &$84.71^{\phantom{\pm0.01}}$                             \\\cline{1-7} \multirow{6}{*} {\shortstack{ImageNet-200 / \\ \\ImageNet-800}}  
    
     &Entropy      &   {${84.92}^{\phantom{\pm0.01}}$}    &${74.75}^{\phantom{\pm0.01}}$                                                    &\textcolor{mycolor4}{${53.62}^{\phantom{\pm0.01}}$}  &\textcolor{mycolor4}{$\underline{89.05}^{\phantom{\pm0.01}}$}   &\textcolor{mycolor4}{$\underline{85.02}^{\phantom{\pm0.01}}$}                              \\  
    
     &EBM (finetune)      &   \textcolor{mycolor3}{$84.14^{\phantom{\pm0.01}}$}     &\textcolor{mycolor3}{$73.31^{\phantom{\pm0.01}}$}                                                    &\textcolor{mycolor3}{$59.73^{\phantom{\pm0.01}}$}  &\textcolor{mycolor4}{$87.54^{\phantom{\pm0.01}}$}     &\textcolor{mycolor4}{${82.81}^{\phantom{\pm0.01}}$}                                      \\  
    
     &DPN      &   $84.87^{\phantom{\pm0.01}}$     &$74.40^{\phantom{\pm0.01}}$                                                    &\textcolor{mycolor3}{$63.84^{\phantom{\pm0.01}}$}  &$87.18^{\phantom{\pm0.01}}$   &\textcolor{mycolor3}{$80.69^{\phantom{\pm0.01}}$}  \\  
    
     &WOODS      &   ${84.99}^{\phantom{\pm0.01}}$     &$\underline{74.98}^{\phantom{\pm0.01}}$                                                    &\textcolor{mycolor4}{${51.71}^{\phantom{\pm0.01}}$}  &\textcolor{mycolor4}{$88.30^{\phantom{\pm0.01}}$}   &\textcolor{mycolor4}{${84.80}^{\phantom{\pm0.01}}$} \\ &SCONE      &  ${84.93}^{\phantom{\pm0.01}}$ &${74.91}^{\phantom{\pm0.01}}$                                                    &\textcolor{mycolor4}{${52.52}^{\phantom{\pm0.01}}$}  &\textcolor{mycolor4}{${88.19}^{\phantom{\pm0.01}}$}   &\textcolor{mycolor4}{${84.50}^{\phantom{\pm0.01}}$}                              \\
     &DUL (ours)      & \cellcolor{gray!20}\textcolor{mycolor4}{$\textbf{85.65}^{\pm0.07}$}  &\cellcolor{gray!20}\textcolor{mycolor4}{$\textbf{75.59}^{\pm0.12}$}                                                    &\cellcolor{gray!20}\textcolor{mycolor4}{$\textbf{49.14}^{\pm0.13}$}  &\cellcolor{gray!20}\textcolor{mycolor4}{$\textbf{89.27}^{\pm0.03}$}   &\cellcolor{gray!20}\textcolor{mycolor4}{$\textbf{85.62}^{\pm0.03}$}                              \\  
    
       \bottomrule
\end{tabular}}
\vskip -0.2in
\end{center}
\end{table}

\subsection{Experimental Setup}
Our settings follow the common practice~\cite{liu2020energy,ming2022poem,yang2022openood,zhang2023openood} in OOD detection. Here we present a brief description and more details about datasets, metrics, and implementation are in Appendix~\ref{appendix:dataset} and~\ref{appendix:implementation}.

\textbf{Datasets.} $\circ$ \textbf{ID datasets $P^{\rm ID}$.} We train the model on different ID datasets including CIFAR-10, CIFAR-100 and ImageNet-200 (a subset of ImageNet-1K~\cite{russakovsky2015imagenet} with 200 classes). $\circ$ \textbf{Auxiliary OOD datasets $P^{\rm SEM}_{\rm train}$.} In CIFAR experiments, we use ImageNet-RC as $P^{\rm SEM}_{\rm train}$. ImageNet-RC is a down-sampled variant of the original ImageNet-1K which is widely adopted in previous OOD detection works~\cite{liu2020energy,ming2022poem,chen2021atom}. We also conduct experiments on the recent TIN-597~\cite{yang2022openood} as an alternative. When ImageNet-200 is ID, the remaining 800 classes termed ImageNet-800 are considered as $P^{\rm SEM}_{\rm train}$. $\circ$ \textbf{OOD detection test sets $P^{\rm SEM}_{\rm test}$} are a suite of diverse datasets introduced by commonly used benchmark~\cite{zhang2023openood}. In CIFAR experiments, we use SVHN~\cite{netzer2011reading}, Places365~\cite{zhou2017places}, Textures~\cite{cimpoi2014describing}, LSUN-R, LSUN-C~\cite{yu15lsun} and iSUN~\cite{pan2015end} as $P^{\rm SEM}_{\rm test}$. When $P^{\rm ID}$ is ImageNet-200, $P^{\rm SEM}_{\rm test}$ consists of iNaturlist~\cite{van2018inaturalist}, Open-Image~\cite{kuznetsova2020open}, NINCO~\cite{bitterwolf2023ninco} and SSB-Hard~\cite{vaze2022semantic}. It is worth noting that in standard OOD detection settings, there should be no overlapped classes between $P^{\rm ID}$, $P^{\rm SEM}_{\rm train}$ and $P^{\rm SEM}_{\rm test}$, otherwise OOD detection is a trivial problem. $\circ$ \textbf{OOD generalization test sets $P^{\rm COV}$} is the original ID test set corrupted with additive Gaussian noise of $\mathcal{N}(0,5)$, following~\cite{bai2023feed}.

\textbf{Metrics.} For OOD detection performance evaluation, we report the average FPR95, AUROC and AUPR to be consistent with~\cite{ming2022poem}. OOD generalization ability is compared in terms of classification accuracy (OOD-Acc). Besides, we also report classification accuracy on ID test sets (ID-Acc).

\textbf{Compared methods.} We compare DUL with a board line of OOD detection methods, including auxiliary OOD required and auxiliary OOD free methods. $\circ$ \textbf{Auxiliary OOD-free methods} do not require $P^{\rm SEM}_{\rm train}$ during training, including MSP~\cite{hendrycks2016baseline}, Maxlogits~\cite{hendrycks2019scaling}, pretrained EBM~\cite{liu2020energy} and Mahalaobis~\cite{lee2018simple}. $\circ$ \textbf{Auxiliary OOD-required methods} explicitly regularize the model on $P^{\rm SEM}_{\rm train}$, including entropy-regularization (Entropy)~\cite{hendrycks2018deep}, finetuned EBM~\cite{liu2020energy}, DPN of Bayesian framework~\cite{malinin2018predictive}, POEM~\cite{ming2022poem} and WOODS~\cite{katz2022training}. We also compare our DUL to recent advanced SCONE~\cite{bai2023feed} which aims to keep a balance between OOD detection and generalization.

\subsection{Experimental Results}
\textbf{Dilemma between OOD detection and generalization (Q1).} We validate the dilemma mentioned before in Fig.~\ref{fig:coverfigure}. As shown in Tab.~\ref{tab:main}, though many advanced methods establish superior OOD detection performance, their OOD generalization degrades a lot. For example, recent SOTA POEM achieves nearly perfect OOD detection performance on CIFAR10 when ImageNet-RC serves as $P^{\rm SEM}_{\rm train}$ with $3.32\%$ false positive error rate (FPR95). However, its OOD-Acc drops a lot (about $10\%$) compared to baseline MSP. This phenomenon is also observed in other advanced methods. To further detail this phenomenon, we reduce the weight of OOD detection regularization terms in Entropy and finetuned EBM and show the performance on both OOD detection and generalization. As shown in Table~\ref{tab:newbaseline}, when the regularization strength increases, OOD detection performance improves (lower FPR.), while the OOD generalization performance degrades (higher error rate).

\textbf{OOD detection and generalization ability (Q2).} As shown in Tab.~\ref{tab:main}, DUL establishes strong overall performance in terms of both OOD detection and generalization. We highlight a few essential observations: 1) \textbf{Compared to auxiliary OOD free methods}, DUL establishes substantial improvement due to additional regularization on auxiliary outliers. 2) \textbf{Compared to auxiliary OOD required methods}, our method achieves superior OOD detection performance without sacrificing generalization ability. Meanwhile, previous OOD detection methods commonly exhibit severely degraded classification accuracy, with many cases increasing by more than $10\%$ error rate. 3) \textbf{Comparison to the most related work SCONE~\cite{bai2023feed}.} Despite recent advanced SCONE simultaneously considering both two targets, we observe that it can be hard to find a good trade-off. In contrast, dual-optimal OOD detection and generalization performance is achieved by our DUL. Noted that DUL is the only method that achieves \textit{state-of-the-art} detection performance (mostly the best or second best) without degraded generalization ability (no red values in the entire row). The sensitive-robust dilemma is no longer observed in our method. These observations justify our expectation of DUL. 4) \textbf{Combining with existing methods.} Besides, to further demonstrate the effectiveness of the proposed DUL, we also add the unchanged overall uncertainty term in Eq.12 to the original Entropy and finetuned EBM. The results in Table~\ref{tab:newbaseline2} show that DUL regularization can also benefit EBM. However, combining Entropy with our regularization can not improve the accuracy substantially. This is not surprising, since the target of Entropy (high entropy prediction) and our DUL (non-increased entropy) directly conflict according to Theorem 1. 5) \textbf{Comparison to methods with an extra OOD detect branch.} Different from aforementioned methods, a line of recent OOD detectors~\cite{miao2024out,bitterwolf2022breaking,liu2024category,chen2021atom} employ extra output branches aside from the classification logits (with a shared backbone for feature extraction). For these OOD detectors, our theoretical analysis is not directly applicable and further analysis from a feature learning perspective may be needed in future work. However, the proposed DUL is devised in a finetune manner. Compared to OOD detectors with extra output branches that requires re-training the classifier from scratch, DUL can be applied to any pre-trained model (e.g., from torchvision, huggingface), with modest compute overhead.

\textbf{Visualization of estimated uncertainty (Q3).} To evaluate the uncertainty estimation, we visualize the distribution of ID (CIFAR-10) and OOD (SVHN) samples in terms of uncertainty. As we can see in Fig.~\ref{fig:uncertainty} (b), our DUL establishes a distinguishable (distributional) uncertainty gap between test-time ID and OOD data, which indicates a good sensitiveness for OOD detection. By contrast, the baseline method MSP (Fig.\ref{fig:uncertainty} (a)) can not effectively discriminate ID and OOD. Besides, we visualize the predictive entropy (overall uncertainty) on covariate-shifted OOD (CIFAR-10 with Gaussian noise) in Fig.~\ref{fig:uncertainty} (c), our DUL yields much lower entropy compared to other methods. Besides, we visualize the data uncertainty on semantic OOD test data (Textures) when CIFAR-10 is ID in Fig.~\ref{fig:uncertainty2}. The investigated methods are 1) pretrained model training on ID dataset only, 2) finetuned model with OOD detection regularization (ablating the last term in Eq.12), and 3) finetuned model with the full DUL method described by Eq.12. As shown in Fig.~\ref{fig:uncertainty2}, to keep the overall uncertainty and enlarge the distributional uncertainty (for OOD detection), the data uncertainty must be reduced. We use Eq.17 from~\cite{malinin2018predictive} to calculate data uncertainty. The distributional uncertainty is shifted by subtracting that on ID dataset. These results meet our expectation.

\begin{figure}[htbp]
\centering
\vspace{0mm}
\includegraphics[width=0.99\textwidth]{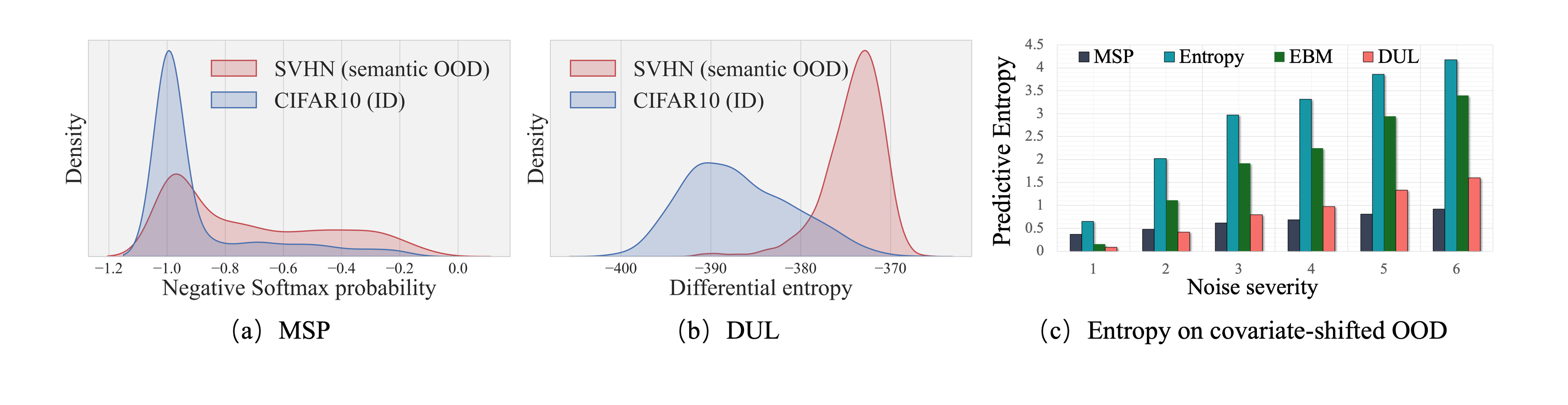}
\vspace{-0mm}
\caption{Visualization of different types of uncertainty estimated by DUL.}
\vspace{-0mm}
\label{fig:uncertainty}
\end{figure}

\begin{figure}[htbp]
\vskip -0.2in
\label{fig:uncertainty2}
    \caption{Visualization of different types of uncertainty on semantic OOD test dataset (i.e., Textures) when CIFAR-10 is ID dataset. Without DUL (orange), all three types of uncertainty will increase altogether on OOD. In contrast, DUL (green) increases the distributional uncertainty but decreases the data uncertainty on OOD, which further lead to unchanged overall uncertainty.}
    \includegraphics[width=\textwidth]{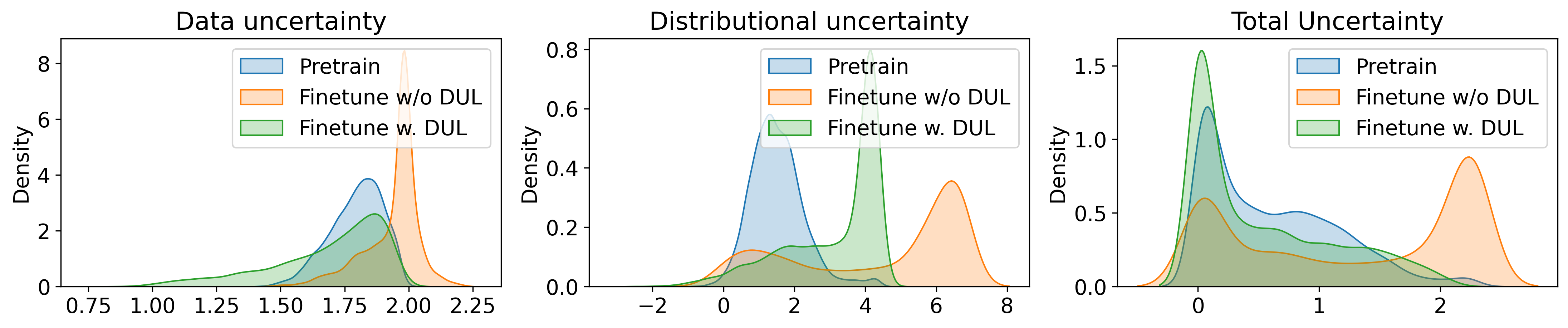}
    \vskip -0.2in
\end{figure}

\begin{multicols}{2}
\noindent
\begin{table}[H]
\caption{Additional results when equip DUL to existing methods i.e., Entropy and finetuned EBM. ID dataset is CIFAR-10. $P^{\rm SEM}_{\rm train}$ is ImageNet-RC. $P^{\rm Cov}_{\rm test}$ is the original CIFAR-10 testset corrupted by Gaussian noise $\mathcal{N}(0,5)$. }
\begin{minipage}{\columnwidth}
\vskip 0pt
\label{tab:newbaseline2}
\center
\resizebox{0.9\columnwidth}{!}{
\setlength{\tabcolsep}{2mm}
\begin{tabular}{c|cc|cc}
\toprule     &  \multicolumn{2}{c|}{Model generalization}                                                    &\multicolumn{2}{c}{OOD detection}\\
   \multirow{1}{*}{Method} &\text{ID-Acc $\uparrow$} & \text{OOD-Acc $\uparrow$} & \text{FPR $\downarrow$}& \text{AUC $\uparrow$} \\ \midrule 
     Entropy      &   $96.04$    & \textcolor{mycolor3}{$72.57$}                                                    &\textcolor{mycolor4}{$6.63$}  &\textcolor{mycolor4}{$98.72$}                           \\
    
     EBM (finetune)      &   $96.10$     &\textcolor{mycolor3}{$79.03$}                                                    &\textcolor{mycolor4}{$3.61$}  &\textcolor{mycolor4}{$98.39$}                               \\  
    
     POEM      &   \textcolor{mycolor3}{$94.32$}     &\textcolor{mycolor3}{$78.89$}                                                    &\textcolor{mycolor4}{$3.32$}  &\textcolor{mycolor4}{$\underline{98.99}$}                                              \\ \midrule 
     EBM w. DUL      &  \cellcolor{gray!20}${95.19}$  &\cellcolor{gray!20}{${87.45}$}                                                    &\cellcolor{gray!20}\textcolor{mycolor4}{${6.17}$}  &\cellcolor{gray!20}\textcolor{mycolor4}{$98.28$}    \\  
     Entropy w. DUL      &  \cellcolor{gray!20}${96.10}$  &\cellcolor{gray!20}{${87.41}$}                                                    &\cellcolor{gray!20}\textcolor{mycolor4}{${29.56}$}  &\cellcolor{gray!20}\textcolor{mycolor4}{$95.92$}    \\  \midrule
      DUL      & \cellcolor{gray!20}$96.02$
 &\cellcolor{gray!20}\textcolor{mycolor4}{$88.01$}                                                    &\cellcolor{gray!20}\textcolor{mycolor4}{$5.89$}  &\cellcolor{gray!20}\textcolor{mycolor4}{$98.47$}   
     \\  
     DUL\textsuperscript{\dag}      &  \cellcolor{gray!20}${96.04}$  &\cellcolor{gray!20}{${87.53}$}                                                    &\cellcolor{gray!20}\textcolor{mycolor4}{${5.99}$}  &\cellcolor{gray!20}\textcolor{mycolor4}{$98.28$}   
     \\  
      \bottomrule

     \end{tabular}}
\vskip -0.5in
\end{minipage}
\end{table}
\columnbreak
\noindent
\begin{table}[H]
\caption{We tune the weight of OOD detection regularization term for EBM as well as Entropy and report the FPR (OOD detection metric) and error rate (Err, OOD generalization metric). The experimental settings are the same with Table 2.}
\begin{minipage}{0.9\columnwidth}
\vskip 15pt
\label{tab:newbaseline}
\center
\resizebox{\columnwidth}{!}{
\setlength{\tabcolsep}{1mm}
\begin{tabular}{ccc|ccc}
\toprule      \multicolumn{3}{c|}{Entropy}  &                                                  \multicolumn{3}{c}{EBM}\\
   \multirow{1}{*}{$\lambda$} &\text{OOD-Err $\downarrow$}&\text{FPR $\downarrow$} & $\lambda$ &\text{OOD-Err $\downarrow$} & \text{FPR $\downarrow$} \\ \midrule 
     $0$      &   $9.55$    & $35.15$                                                    &$0$  &$9.55$  &$20.57$                      \\
    
     $5\times 10^{-4}$      &   $13.58$     &$8.36$                                                    &$1\times 10^{-4}$  &$9.46$   &$14.69$                            \\  
    
     $5\times 10^{-3}$      &   $15.48$     &$6.37$                                                    &$1\times 10^{-3}$  &$10.32$   &$13.54$                                           \\$5\times 10^{-2}$      &   $17.97$     &$5.71$                                                    &$1\times 10^{-2}$  &$16.43$  &$8.15$   \\  
     $5\times 10^{-1}$      &  ${18.53}$  &${5.60}$                                                    &${1\times 10^{-1}}$  &{$24.38$}  &$6.11$   \\ \bottomrule

     \end{tabular}}
\vskip -0.5in
\end{minipage}
\end{table}
\end{multicols}

\section{Conclusion}
This paper provides both theoretical and empirical analysis towards understanding the dilemma between OOD detection and generalization. We demonstrate that the superior OOD detection performance of current advances are achieved at the cost of generalization ability. The theory-inspired algorithm successfully removes the conflict between previous OOD detection and generalization methods. For SOTA OOD detection performance, our implementation assumes that auxiliary outliers are available during training. This limitation is noteworthy for our DUL as well as the most existing SOTA OOD detection methods. We argue that this added cost is minor and reasonable given the significance of ensuring model trustworthiness in open-environments. Reducing the dependency on auxiliary OOD data can be an interesting research direction for the future exploration.


\clearpage 
\newpage
\appendix

\hypersetup{ colorlinks = false, pdfborder = {0 0 0},
    linkbordercolor = {1 1 1} }

\startcontents[sections]
\printcontents[sections]{l}{1}{\section*{Appendices}\setcounter{tocdepth}{2}}

\section{Proofs}
\label{appendix:proof}
First, we recap the definitions of Disparity with Total Variation Distance and Disparity Discrepancy.
\begin{definition}[Disparity with Total Variation Distance] Given two hypotheses $f',f\in\mathcal{F}$ and distribution $P$, we define the Disparity with Total Variation Distance between them as
\begin{equation}
    {\rm disp}_P(f',f)=\mathbb{E}_P[TV(F_f||F_{f'})],
\end{equation}
where $F_f,F_{f'}$ are the class distributions predicted by $f',f$ respectively. $TV(\cdot||\cdot)$ is the total variation distance, i.e., $TV(F_f||F_{f'})=\frac{1}{2}\sum_{k=1}^K||F_{f,k}-F_{f',k}||$. 
\end{definition}

\begin{definition}[Disparity Discrepancy with Total Variation Distance, DD with TVD] Given a hypothesis space $\mathcal{F}$ and two distributions $P,Q$, the Disparity Discrepancy with Total Variation Distance (DD with TVD) is defined as
\begin{equation}
    d_{\mathcal{F}}(P,Q):=\underset{f',f\in\mathcal{F}}{\rm sup}({\rm disp}_{P}(f',f)-{\rm disp}_{Q}(f',f)).
\end{equation}
\end{definition}

Since TVD is a distance measurement of two distribution. It yields the triangle equality. For any distribution $P_{\mathcal{X}}$ support on $\mathcal{X}$ and hypotheses $f^1,f^2$ and $f^3$ $\in\mathcal{F}$, we have 
\begin{equation}
\begin{split}
    {\rm disp}_{P_\mathcal{X}}(f^1,f^2)\leq \mathbb{E}_{x\sim{P_\mathcal{X}}}[TV(F_{f^1}(x)||F_{f^3}(x))]+\mathbb{E}_{x\sim{P_\mathcal{X}}}[TV(F_{f^2}(x)||F_{f^3}(x))]
    \\
    {\rm disp}_{P_\mathcal{X}}(f^1,f^2)\geq \mathbb{E}_{x\sim P_\mathcal{X}}[TV(F_{f^1}(x)||F_{f^3}(x))]-\mathbb{E}_{x\sim P_\mathcal{X}}[TV(F_{f^2}(x)||F_{f^3}(x))]
\end{split}
\end{equation}
where $TV(\cdot||\cdot)$ is total variance distance (TVD).

To prove Theorem~\ref{Theorem 1}, we need the following two lemmas.
\begin{lemma}
\label{Lemma 1} For any $f\in\mathcal{F}$, we have
\begin{equation}
\begin{split}
    \mathbb{E}_{P^{\rm COV}}TV(F_f||\mathcal{U})\leq \mathbb{E}_{P^{\rm SEM}_{\rm test}}TV(F_f||\mathcal{U})+d_\mathcal{F}(P^{\rm COV},P^{\rm SEM}_{\rm test})+\lambda
\end{split}
\end{equation}
where $\lambda$ is a constant independent of $f$. $\mathcal{U}$ is the $K$-classes uniform distribution. $P^{\rm COV}$ is the covariant-shifted OOD distribution. $P^{\rm SEM}_{\rm test}$ is the semantic OOD distribution.
\end{lemma}

\begin{proof}
Let $f^*$ be the hypothesis which jointly minimizes the total variance distance between the predicted distribution $F_f$ with uniform distribution $\mathcal{U}$ taking expectation on $P^{\rm COV}, P^{\rm SEM}_{\rm test}$ , which is to say
\begin{equation}
    f^*:=\underset{f\in\mathcal{F}}{\rm argmin}\{\mathbb{E}_{x\sim{P^{\rm COV}}}[TV(F_{f}(x)||\mathcal{U})]+\mathbb{E}_{x\sim{P^{\rm SEM}_{\rm test}}}[TV(F_{f}(x)||\mathcal{U})]\}.
\end{equation}
Set $\lambda=\mathbb{E}_{x\sim{P^{\rm COV}}}[TV(F_{f^*}(x)||\mathcal{U})]+\mathbb{E}_{x\sim{P^{\rm SEM}_{\rm test}}}[TV(F_{f^*}(x)||\mathcal{U})]$, then by the triangle equality, we have
\begin{equation}
\begin{split}
   \mathbb{E}_{P^{\rm COV}} TV(F_{f}||\mathcal{U})&\leq {\rm disp}_{{P^{\rm COV}}}(f,f^*)+\mathbb{E}_{P^{\rm COV}}TV(F_{f^*}||\mathcal{U})
\\
&\leq\mathbb{E}_{P^{\rm SEM}_{\rm test}}TV(F_{f}||\mathcal{U})-\mathbb{E}_{P^{\rm SEM}_{\rm test}}TV(F_{f}||\mathcal{U})+{\rm disp}_{{P^{\rm COV}}}(f,f^*)+\mathbb{E}_{P^{\rm COV}}TV(F_{f^*}||\mathcal{U})
\\
&\leq\mathbb{E}_{P^{\rm SEM}_{\rm test}}TV(F_{f}||\mathcal{U})+\mathbb{E}_{P^{\rm SEM}_{\rm test}}TV(F_{f^*}||\mathcal{U})
\\
&\ \ -{\rm disp}_{P^{\rm SEM}_{\rm test}}(f,f^*)+{\rm disp}_{P^{\rm COV}}(f,f^*)+\mathbb{E}_{P^{\rm COV}}TV(F_{f^*}||\mathcal{U})
\\
&\leq \mathbb{E}_{P^{\rm SEM}_{\rm test}}TV(F_{f}||\mathcal{U})+d_{\mathcal{F}}(P^{\rm SEM}_{\rm test},P^{\rm COV})+\lambda.
\end{split}
\end{equation}
\end{proof}
Intuitively speaking, Lemma~\ref{Lemma 1} demonstrates that if the classifier $f$ express high overall uncertainty on $P^{\rm SEM}_{\rm test}$ (the predicted distribution $F_f$ is close to uniform distribution), it will also tend to high uncertain prediction on $P^{\rm COV}$ given a limited $d_{\mathcal{F}}(P^{\rm SEM}_{\rm test},P^{\rm COV})$.
\begin{lemma}
\label{Lemma 2} Denote OOD detection loss used for MSP detectors as $\mathcal{L}_{\rm MSP}$, then we have
\begin{equation}
\mathbb{E}_{P_{\rm test}^{\rm SEM}}TV(F_f||\mathcal{U})\leq \mathbb{E}_{P_{\rm test}^{\rm SEM}}\sqrt{\frac{1}{2}(\mathcal{L}_{{\rm MSP}}(f) - \log K)}
\end{equation}
where $TV(\cdot||\cdot)$ is total variance distance (TVD), $\mathcal{U}$ denotes uniform distribution support on $\mathcal{Y}=\{1,2\cdots K\}$. $\mathcal{L}_{\rm MSP}$ is defined in ~\cite{hendrycks2018deep} is the cross-entropy between predicted distribution $F_f(x)$ and uniform distribution $\mathcal{U}$.
\end{lemma}
Lemma~\ref{Lemma 2} means that minimizing the OOD detection loss will constrains the predicted distribution to be close to uniform distribution, which is a straightforward result.
\begin{proof}
In $K$-classes classification task, for any sample $\tilde{x}$ drawn from $P_{\mathcal{X}}^{\rm SEM}$, we have
\begin{equation}
    \mathcal{L}_{{\rm OE}}(f(\tilde{x}))= KL(\mathcal{U}||F_f)+H(\mathcal{U}).
\end{equation}
Applying Pinsker's inequality, the following inequality holds
\begin{equation}
    \mathcal{L}_{{\rm OE}}(f(\tilde{x}))= KL(\mathcal{U}||F_f)+H(\mathcal{U})\geq2TV(\mathcal{U}||F_f(\tilde{x}))^2+H(\mathcal{U}).
\end{equation}
Noted that $H(\mathcal{U})=\log K$, we can written above inequality as
\begin{equation}
    \mathcal{L}_{{\rm OE}}(f(\tilde{x}))\leq \frac{1}{2}\sqrt{KL(\mathcal{U}||F_f(\tilde{x}))-\log K}.
\end{equation}
Then, by taking expectation on $P_{\mathcal{X}}^{\rm SEM}$ we can get the result.
\end{proof}
Now we are ready to present the proof of Theorem~\ref{Theorem 1}.
\begin{proof}
By the definition, the generalization error can be written as
\begin{equation}
\begin{split}
        {\rm GError}_{P_{\mathcal{XY}}^{\rm COV}}(f):=&\mathbb{E}_{(x,y)\sim P_{\mathcal{XY}}^{\rm COV}}\mathcal{L}_{\rm CE}(f(x),y)
    \\
=&\mathbb{E}_{(x,y)\sim P_{\mathcal{XY}}^{\rm COV}}KL(p(y|x)||F_f(x))+H(p(y|x))
\end{split}
\end{equation}
where $p(y|x)$ is the target distribution given input $x$ (ground truth).

Applying Bretagnolle–Huber inequality, for any $x$ we have 
\begin{equation}
\begin{split}
        KL(p(y|x)||F_f(x))+1\geq&2\sqrt{KL(p(y|x)||F_f(x))}
        \\
        \geq& 2\sqrt{1-\exp^{-KL(p(y|x)||F_f(x))}}
        \\
        \geq& 2TV(p(y|x)||F_f(x)).
\end{split}
\end{equation}
By the sub-additivity of TVD, we have
\begin{equation}
\begin{split}
        KL(p(y|x)||F_f(x))\geq&2TV(p(y|x)||F_f(x))-1
        \\
        \geq&2[TV(p(y|x)||\mathcal{U})-TV(F_f(x)||\mathcal{U})]-1.
\end{split}
\end{equation}
Taking expectation on $P_{\mathcal{XY}}^{\rm COV}$, we have
\begin{equation}
    \begin{split}
        {\rm GError}_{P_{\mathcal{XY}}^{\rm COV}}(f)
=&\mathbb{E}_{(x,y)\sim P_{\mathcal{XY}}^{\rm COV}}KL(p(y|x)||F_f(x))+H(p(y|x))
\\
\geq&2\mathbb{E}_{(x,y)\sim P_{\mathcal{XY}}^{\rm COV}} [TV(p(y|x)||\mathcal{U})-TV(F_f(x)||\mathcal{U})]-1+H(p(y|x))
    \end{split}
\end{equation}
Applying Lemma~\ref{Lemma 1} and Lemma~\ref{Lemma 2},
\begin{equation}
    \begin{split}
        {\rm GError}_{P_{\mathcal{X}}^{\rm COV}}(f)
=&\mathbb{E}_{P_{\mathcal{XY}}^{\rm COV}}KL[p(y|x)||F_f(x)]+H(p(y|x))
\\
\geq&2\mathbb{E}_{P_{\mathcal{XY}}^{\rm COV}} [TV(p(y|x)||\mathcal{U})-TV(F_f(x)||\mathcal{U})]-1+H(p(y|x))
\\
\geq&2\mathbb{E}_{P_{\mathcal{XY}}^{\rm COV}}TV(p(y|x)||\mathcal{U})-2\mathbb{E}_{P_{\mathcal{XY}}^{\rm COV}}TV(F_f||\mathcal{U})-1+H(p(y|x))
\\
\geq&2\mathbb{E}_{P_{\mathcal{XY}}^{\rm COV}}TV(p(y|x)||\mathcal{U})-2\mathbb{E}_{P_{\mathcal{X}}^{\rm SEM}}TV(F_f||\mathcal{U})-2d_{\mathcal{F}}(P_{\mathcal{XY}}^{\rm COV},P_{\mathcal{X}}^{\rm SEM})-2\lambda-1+H(p(y|x))
\\
\geq&2\mathbb{E}_{P_{\mathcal{XY}}^{\rm COV}}TV(p(y|x)||\mathcal{U})-\mathbb{E}_{P_{\mathcal{X}}^{\rm SEM}}\sqrt{2(\mathcal{L}_{{\rm OE}}(f) - \log K)}
\\& \ -2d_{\mathcal{F}}(P_{\mathcal{X}}^{\rm COV},P_{\mathcal{X}}^{\rm SEM})-2\lambda-1+H(p(y|x))
\\
\geq&-\mathbb{E}_{P_{\mathcal{X}}^{\rm SEM}}\sqrt{2(\mathcal{L}_{{\rm OE}}(f) - \log K)}-2d_{\mathcal{F}}(P_{\mathcal{XY}}^{\rm COV},P_{\mathcal{X}}^{\rm SEM})+C.
    \end{split}
\end{equation}
Given the fact that $P^{\rm COV}_{\mathcal{XY}}$, $p(y|x)$ are both fixed, $H(p(y|x))$ and $TV(p(y|x)||\mathcal{U})$ are constants for each $x$. Finally, we get
\begin{equation}
    {\rm GError}_{P_{\mathcal{XY}}^{\rm COV}}(f)
\geq C-\mathbb{E}_{P_{\mathcal{X}}^{\rm SEM}}\sqrt{2(\mathcal{L}_{{\rm OE}}(f) - \log K)}-2d_{\mathcal{F}}(P_{\mathcal{X}}^{\rm COV},P_{\mathcal{X}}^{\rm SEM}).
\end{equation}
\end{proof}

\section{Experimental Details}

\subsection{Datasets details}
\label{appendix:dataset}
\textbf{ID datasets $P^{\rm ID}$.} ID datasets are chosen following common practice in OOD detection. We use CIFAR-10, CIFAR-100 and ImageNet-200 as $P^{\rm ID}$. ImageNet-200 is a subset of the original ImageNet-1K introduced by~\cite{yang2022openood,zhang2023openood}.

\textbf{Auxiliary OOD datasets $P^{\rm SEM}_{\rm train}$.}
For CIFAR experiments, we use ImageNet-RC and TIN-597 as auxiliary datasets. ImageNet-RC is a down-sampled variant of the ImageNet-1K, which consists of 1000 classes and 1,281,167 images. We also conduct experiments on TIN-597 as an alternative for ImageNet-RC. TIN-597 is introduced by recent work~\cite{zhang2023openood}. The resolutions of ID and auxiliary samples are both $64\times 64$. For ImageNet experiments, we use a subset of ImageNet-1K consisting of 200 classes as ID datasets. The remaining images belong to other 800 classes are utilized as auxiliary datasets. The resolutions of ID and auxiliary images are both $224 \times 224$.

\textbf{OOD detection test datasets $P^{\rm SEM}_{\rm test}$.}
In CIFAR experiments, following standard practice~\cite{liu2020energy}, we use SVHN~\cite{netzer2011reading}, Textures~\cite{cimpoi2014describing}, Places365~\cite{zhou2017places}, iSUN~\cite{pan2015end}, LSUN-C and LSUN-R~\cite{yu15lsun} to evaluate the OOD detection performance. $\circ$ The SVHN test set comprises 26,032 color images of house numbers. $\circ$ Textures (Describable Textures Dataset, DTD) consists of 5,640 images depicting natural textures. $\circ$ Places365 dataset consists scenic images of 365 different categories. Each class consists of 900 images. $\circ$ The iSUN dataset is a subset of the SUN database with 8,925 images. $\circ$ The Large-scale Scene Understanding dataset (LSUN) comprises a testing set with 10,000 images of 10 different scenes. LSUN offers two datasets, LSUN-C and LSUN-R. In LSUN-C, the original high-resolution images are randomly cropped into \(32 \times 32\). Meanwhile, in LSUN-R, the images are resized to \(32 \times 32\). In ImageNet experiments, we follow the settings of~\cite{zhang2023openood}, where OpenImage-O~\cite{wang2022vim}, SSB-hard~\cite{vaze2021open}, Textures~\cite{cimpoi2014describing}, iNaturalist~\cite{van2018inaturalist} and NINCO~\cite{bitterwolf2023ninco} are selected as OOD detection test datasets. $\circ$ OpenImage-O contains 17632 manually filtered images and is 7.8 \( \times \) larger than the ImageNet dataset. $\circ$ SSB-hard is selected from ImageNet-21K. It consists of 49K images and 980 categories. $\circ$ Textures (Describable Textures Dataset, DTD) consists of 5,640 images depicting natural textures. $\circ$ iNaturalist consists of 859000 images from over 5000 different species of plants and animals. $\circ$ NINCO consists with a total of 5879 samples of 64 classes which are non-overlapped with ImageNet-1K.

\textbf{OOD generalization test datasets $P^{\rm COV}$.} Following previous work~\cite{bai2023feed}, we corrupt the original test data with Gaussian noise of zero mean and variance of 5 in the main paper. In appendix, we conduct additional experiments involving CIFAR10-C, CIFAR100-C and ImageNet-C~\cite{hendrycks2018benchmarking} with 15 different types of noise.

\subsection{Implementation details}
\label{appendix:implementation}
\subsubsection{CIFAR experiments.} We use WideResNet-40-10~\cite{zagoruyko2016wide} as the backbone network, which comprises 40 layers. The widen factor is set to 10. We use SGD optimizer to train all methods with dropout strategy. The dropout rate is 0.3. The momentum is set to 0.9 and weight decay is set to 0.0005.

\textbf{Pretraining details.}
The pretrained model is obtained by training WideResNet-40-10 for 200 epochs with an initial learning rate of 0.1. We decay the learning rate by a factor of 0.2 at the 60-th, 120-th, and 160-th epochs. Batch size is set to 128.

\textbf{Finetuning details.}
 $\circ$ For Entropy and EBM, we finetune the pretrained model for 20 epochs with an initial learning rate of 0.001, utilizing a cosine annealing strategy to adjust the learning rate. Following the official implementation, the weight of OOD detection regularization term is set to 0.5 and 0.1 for Entropy and EBM (finetune) respectively. The hyperparameters $m_{\rm ID}$ and $m_{\rm OOD}$ in EBM regularization learning are set to -25 and -7 respectively. The ID batch size is 128 and the OOD batch size is set to 256.  $\circ$ For SCONE, we finetune the pretrained model for 10 epochs with an initial learning rate of 0.0002, utilizing a cosine annealing strategy to adjust the learning rate. The batch size is 32, the OOD batch size is 64. The margin of the OOD detection boundary is set to 1. To be aligned with most previous works in OOD detection and generalization, we assume $P_{\mathcal{X}}^{\rm COV}$ is unavailable during finetuning.  $\circ$ For WOODS, we finetune the pretrained model for 10 epochs with an initial learning rate of 0.0002, utilizing a cosine annealing strategy to adjust the learning rate. The ID batch size is 32, the OOD batch size is 64. Other hyperparameters settings are consistent with SCONE. $\circ$ For DUL, $\alpha_0$ is set to 12. While finetuning on CIFAR10, the $m_{\rm ID}$ and $m_{\rm OOD}$ are set to 10 and 30 respectively. The weight $\lambda,\gamma$ are set to 0.3 and 2. We train for 20 epochs with an initial learning rate of 0.00005, utilizing a cosine annealing strategy to adjust the learning rate. While finetuning on CIFAR100/TIN-597, the $m_{\rm ID}$ and $m_{\rm OOD}$ are set to 10 and 30 respectively. The weights $\lambda,\gamma$ are set to 0.05 and 2 respectively. We finetune for 20 epochs with an initial learning rate of 0.00005, utilizing a cosine annealing strategy to adjust the learning rate. While finetuning on CIFAR100/ImageNet-RC, we set $h_0=0$. The $m_{\rm ID}$ and $m_{\rm OOD}$ are set to -430 and -370 respectively. We train for 30 epochs with an initial learning rate of 0.0001, utilizing a cosine annealing strategy to adjust the learning rate. The weights $\lambda,\gamma$ are set to 0.1 and 1 respectively. For CIFAR-100/ImageNet-RC, we set $\tau=2$ and otherwise $\tau=1$.$\circ$ For DUL\textsuperscript{\dag}, we use Thompson sampling strategy~\cite{ming2022poem} for OOD informativeness mining. The sampling hyperparameters are consistent with that of POEM.

\textbf{Training from scratch details.}
 $\circ$ For POEM, we train from scratch for 200 epochs with an initial learning rate of 0.1, and decay the learning rate by a factor of 0.2 at the $60$-th, $120$-th, and $160$-th epochs. The ID and OOD batch size are set to 128 and 256 respectively. Following the official implementation, the pool of outliers consists of randomly selected 400,000 samples from auxiliary datasets, and only 50,000 samples (same size as the ID training set) are selected for training based on the boundary score.  $\circ$ For DPN, we train for 200 epochs with an initial learning rate of 0.1, and decay the learning rate by a factor of 0.2 at the 60-th, 120-th, and 160-th epochs. The Dirichlet parameters $\mathbf{\alpha}$ are calculated by performing ReLU plus one on the model's outputs, i.e., $\alpha={\rm ReLU}(f(x))+1$. $\alpha_0$ is set to 15 and 12 respectively when training on CIFAR10 and CIFAR100. The auxiliary datasets are ImageNet and TIN-597. The ID and OOD batch size are set to 128 and 256 respectively. When training on CIFAR100/TIN-597, the OOD regularization weight $\lambda$ is set to 0.05. In other cases, $\lambda$ is set to 0.5.

\subsubsection{ImageNet experiments.}
We use ResNet18~\cite{he2016deep} as the backbone network. We use SGD optimizer to train all the models. The momentum is set to 0.9.

\textbf{Pretraining details.}
The pretrained model is obtained by training ResNet18 for 100 epochs with an initial learning rate of 0.1, utilizing a cosine annealing strategy to adjust the learning rate. The weight decay is set to 0.0001. Batch size is set to 64.

\textbf{Finetuning details.}
$\circ$ For Energy regularized learning, we finetune the pretrained model for 10 epochs with an initial learning rate of 0.001, utilizing a cosine annealing strategy to adjust the learning rate. The weight decay is set to 0.0001. Following the official implementation, the weights of OOD detection regularization term are set to 0.1. Specifically, the $m_{\rm ID}$ and $m_{\rm OOD}$ in energy regularization method are set to -25 and -7 respectively. The ID batch size is 64 and the OOD batch size is set to 128. $\circ$ For Entropy, we finetune the pretrained model for 10 epochs with an initial learning rate of 0.001, utilizing a cosine annealing strategy to adjust the learning rate. The weight decay is set to 0.0001. The ID and OOD batch size are set to 64 and 128 respectively. Following the official implementation, the weights of OOD detection regularization term are set to 0.5.$\circ$ For SCONE, we finetune the pretrained model for 10 epochs with an initial learning rate of 0.0002, utilizing a cosine annealing strategy to adjust the learning rate. The weight decay is set to 0.0005. The batch size is 32, the OOD batch size is 64. The margin of the OOD detection boundary is set to 1. To be aligned with most previous works in OOD detection and generalization, we assume $P_{\mathcal{X}}^{\rm COV}$ is unavailable during finetuning.  $\circ$ For WOODS, we finetune the pretrained model for 10 epochs with an initial learning rate of 0.0002, utilizing a cosine annealing strategy to adjust the learning rate. The batch size is 32, the OOD batch size is 64. Other hyperparameters of WOODS are consistent with SCONE. $\circ$ For our DUL, the Dirichlet parameters $\mathbf{\alpha}$ are calculated by performing ReLU and exp operation on the model's outputs, i.e., $\alpha=\exp({\rm ReLU}(f(x)))$. For numerical stability, we measure the distributional uncertainty by the strength of Dirichlet distribution. $\lambda,\gamma$ are set to 0.1 and 4 respectively. We set $\tau=1$ in large-scale experiments.

\textbf{Training from scratch details.}
$\circ$ For DPN, we train ResNet18 for 100 epochs with an initial learning rate of 0.1, utilizing a cosine annealing strategy to adjust the learning rate. The weight decay is set to 0.0001. Batch size is set to 64. The Dirichlet parameters $\mathbf{\alpha}$ are calculated by performing ReLU plus one on the model's outputs, i.e., $\alpha={\rm ReLU}(f(x))+1$. The ID classification loss is set to KL-divergence between predicted class distribution under Dirichlet prior and target distribution because of the inconvenience of directly setting $\alpha_0$. The target distribution is obtained by label smoothing strategy with parameter of 0.01~\cite{malinin2018predictive}. The weight of regularization term applied on OOD auxiliary samples is 1.
 
\begin{algorithm}[ht]
\label{alg-DUL}
\SetKwInOut{Input}{\textbf{Input}}
\SetKwInOut{Output}{\textbf{Output}}
\SetAlgoNoLine
 	\caption{Pseudo Code of Decoupled Uncertainty Learning (DUL)}
 	 	\label{alg:DUL}
 		\Input{ ID data $P^{\text{ID}}$, auxiliary outliers $P^{\text{SEM}}_{\rm train}$, classifier $f_{\theta_0}$ pretrained on $P^{\rm ID}$.}
        \Output{ finetuned classifier $f_{\theta}$}
      Initialize $\theta=\theta_0$;
      
 		\For {each iteration}{
 		    Obtain ID sample $(x, y)$ from $P^{\text{ID}}$ and auxiliary outlier $\tilde{x}$ from $P^{\text{SEM}}_{\rm train}$;
            
            Update model parameters $\theta$ by minimizing objective defined in Eq.~\ref{eq:object-practice};
           }
\end{algorithm}

\section{Additional Results}
\label{appendix:full results}

\subsection{Uncertainty estimation.}
We add Gaussian noise with zero mean and varying variance $\epsilon$ on CIFAR-10 and investigate the estimated distributional uncertainty and overall uncertainty. Distributional uncertainty is measured by differential entropy. It clear that with DUL regularization, the prediction yields a low overall uncertainty and high distributional uncertainty on covariate-shifted data. We conduct experiments on CIFAR-10/ImageNet-RC and CIFAR-10/TIN-597, tabular results are shown in Tab.~\ref{tab:uncertainty-visualization}.

\begin{table}[!ht]
\centering
\small
\vskip 0.15in
\caption{Mean value of estimated uncertainty on CIFAR-10-C with varying severity of Gaussian noise with zero mean and variance of $\epsilon$.}
\label{tab:uncertainty-visualization}
\resizebox{1.0\textwidth}{!}{
\begin{tabular}{ccccccccc}
\toprule
$\mathcal{P}_{\mathcal{X}}^{\rm in}/\mathcal{P}_{\mathcal{X}}^{\rm aux}$&Uncertainty type& DUL & $\epsilon=0.0$ & $\epsilon=2.0$ & $\epsilon=4.0$ & $\epsilon=6.0$ & $\epsilon=8.0$ & $\epsilon=10.0$ \\
\midrule
\multirow{4}{*}{\shortstack{CIFAR-10\\ImageNet-RC}}&\multirow{2}{*}{Distributional uncertainty}& \xmark & -21.33 & -18.94 & -15.62 & -13.71 & -12.91 &-12.85\\
& & \cmark & -21.23 & -19.42 & -16.96 & -15.47 & -14.85 &-14.58\\
\cline{2-9}
& \multirow{2}{*}{Total uncertainty}& \xmark & 0.04 & 0.47 & 1.31 & 1.93 &2.20 & 2.28 \\
& & \cmark & 0.03 & 0.17 & 0.48 & 0.77 & 0.98 &1.14\\
\midrule
\multirow{4}{*}{\shortstack{CIFAR-10\\TIN-597}}&\multirow{2}{*}{Distributional uncertainty}& \xmark & -20.94 & -20.14 & -18.29 & -16.23 & -14.71 &-13.72\\
& & \cmark & -21.48 & -20.68 & -19.21 & -17.70 & -16.57 &-15.78\\
\cline{2-9}
& \multirow{2}{*}{Total uncertainty}& \xmark & 0.06& 0.15 & 0.51 & 1.06 & 1.55 &1.92  \\
& & \cmark & 0.04 & 0.08 & 0.18 & 0.34 & 0.51 &0.68 \\
\bottomrule
\end{tabular}

}
\end{table}

\subsection{Time-consuming comparison.}
We compare the time-cost of proposed DUL to other training-required OOD detection methods in Tab.~\ref{tab:time-comsuming}. We run all the experiments on one single NVIDIA GeForce RTX-3090 GPU. Compared with other OOD detection methods in a finetune manner, DUL does not introduce noticeably extra cost of computation.

\begin{table*}[htbp]
\centering
\caption{Average execution times (s) per epoch of training required OOD detection methods. Compare to other OOD detection methods, DUL does not introduce noticeable computational cost.}
\label{tab:time-comsuming}
\resizebox{1.0\textwidth}{!}{
\begin{tabular}{lcccc}
\toprule
Method & CIFAR-10/ImageNet-RC & CIFAR-10/TIN-597 & CIFAR-100/ImageNet-RC & CIFAR-100/TIN-597 \\
\midrule
EBM (finetune) & 354.87 & 229.12 & 355.70 & 112.47 \\
Entropy & 355.98 & 140.88 & 1108.43 & 120.85 \\
DPN & 842.05 & 75.62 & 841.69 & 80.98 \\
POEM & 615.51 & 483.04 & 825.87 & 500.01 \\
WOODS & 808.77 & 291.37 & 906.67 & 282.50 \\
SCONE & 911.83 & 183.57 & 1169.33 & 160.00 \\
DUL & 329.58 & 101.08 & 925.62 & 100.43 \\
DUL* & 598.68 & 597.82 &544.26 & 595.25 \\
\bottomrule
\end{tabular}}
\end{table*}

\subsection{Full results with standard deviation.}
Full results with standard deviation are presented this section. In CIFAR experiments, we report the mean and standard deviation in 5 random runs. In ImageNet experiments,  we report the mean and standard deviation in 3 random runs to be consist with~\cite{zhang2023openood}. CIFAR experimental results are shown in Tab.~\ref{tab:full-results-with-std}. Large-scale ImageNet results are shown in Tab.~\ref{tab:full-results-with-std on imagenet}.

\begin{table*}[!ht]
\vskip 0.15in
\begin{center}
\caption{OOD detection and generalization performance comparison with standard variance. Marginal \textcolor{mycolor4}{improvement} and \textcolor{mycolor3}{degradation} ($\geq 0.5$) compare to the baseline method MSP are highlighted in blue or red respectively. The \textbf{best} and \underline{second best} results are in bold or underlined. DUL is the only method achieves state-of-art OOD detection performance (mostly the best or second best) without trade-offs on generalization i.e., the value of entire row is either blue or black.}
\label{tab:full-results-with-std}
\center
\resizebox{1.0\textwidth}{!}{
\setlength{\tabcolsep}{3.9mm}
\begin{tabular}{c|c|cc|ccc}
\toprule \multirow{2}{*}{$\mathcal{P}^{\rm ID}/\mathcal{P}^{\rm SEM}_{\rm train}$}   &\multirow{2}{*}{Method}   &  \multicolumn{2}{c|}{ID/OOD generalization}                                                    &\multicolumn{3}{c}{OOD detection}\\
 &    &\text{ID-Acc. $\uparrow$} & \text{OOD-Acc. $\uparrow$} & \text{FPR $\downarrow$}& \text{AUROC $\uparrow$}  &\text{AUPR $\uparrow$}  \\ \midrule \multirow{5}{*}{\shortstack{CIFAR-10\\ \\Only}}
         
      &MSP        &  $96.11^{\pm0.09}$   &  $87.35^{\pm0.58}$                                                  &$41.96^{\pm3.85}$  & $89.28^{\pm1.12}$  & $68.00^{\pm2.19}$                          \\  &EBM (pretrain)      &  $96.11^{\pm0.09}$   &  $87.35^{\pm0.58}$                                                  & $32.45^{\pm3.45}$ & $89.34^{\pm1.21}$   &$75.22^{\pm2.67}$                              \\ &Maxlogits      &  $96.11^{\pm0.09}$   &  $87.35^{\pm0.58}$                                                   &$32.90^{\pm3.51}$ & $89.26^{\pm1.21}$   & $74.47^{\pm2.55}$                            \\ &Mahalanobis     &  $96.11^{\pm0.09}$   &  $87.35^{\pm0.58}$                                                 & $32.53^{\pm9.61}$ & $93.93^{\pm2.68}$   & $74.96^{\pm7.47}$                                                            \\ \midrule \multirow{8}{*} {\shortstack{CIFAR-10\\ \\ImageNet-RC}} 
    
     &Entropy      &   $\underline{96.04}^{\pm0.14}$    & \textcolor{mycolor3}{$72.57^{\pm3.87}$}                                                    &\textcolor{mycolor4}{$6.63^{\pm0.80}$}  &\textcolor{mycolor4}{$\underline{98.72}^{\pm0.14}$}   &\textcolor{mycolor4}{$94.00^{\pm1.00}$}                            \\  
    
     &EBM (pretrain)      &   $\textbf{96.10}^{\pm0.23}$     &\textcolor{mycolor3}{$79.03^{\pm2.53}$}                                                    &\textcolor{mycolor4}{$\underline{3.61}^{\pm0.71}$}  &\textcolor{mycolor4}{$98.39^{\pm0.39}$}   &\textcolor{mycolor4}{$94.88^{\pm0.91}$}                              \\  
    
     &POEM      &   \textcolor{mycolor3}{$94.32^{\pm0.14}$}     &\textcolor{mycolor3}{$78.89^{\pm2.25}$}                                                    &\textcolor{mycolor4}{$\textbf{3.32}^{\pm0.41}$}  &\textcolor{mycolor4}{$\textbf{98.99}^{\pm0.17}$}   & \textcolor{mycolor4}{$\textbf{99.38}^{\pm0.12}$}                 \\ &DPN      &   $95.69^{\pm0.17}$     &\textcolor{mycolor3}{${85.52}^{\pm0.51}$}                                                    &\textcolor{mycolor4}{${4.28}^{\pm0.60}$}  &\textcolor{mycolor4}{${98.53}^{\pm0.17}$}   & \textcolor{mycolor4}{${94.93}^{\pm0.60}$}  \\ &WOODS      &   $96.01^{\pm0.16}$     &\textcolor{mycolor3}{${80.14}^{\pm1.69}$}                                                    &\textcolor{mycolor4}{${7.12}^{\pm1.54}$}  &\textcolor{mycolor4}{${98.40}^{\pm0.21}$}   & \textcolor{mycolor4}{${92.92}^{\pm0.96}$}                            \\         &SCONE      &$95.96^{\pm0.08}$ &\textcolor{mycolor3}{$78.80^{\pm1.57}$}                                                    &\textcolor{mycolor4}{$7.02^{\pm1.06}$}  &\textcolor{mycolor4}{$98.45^{\pm0.12}$}   &\textcolor{mycolor4}{$92.46^{\pm0.93}$}                             \\
       &DUL (ours)      & \cellcolor{gray!20}$96.02^{\pm0.07}$
 &\cellcolor{gray!20}\textcolor{mycolor4}{$\textbf{88.01}^{\pm0.54}$}                                                    &\cellcolor{gray!20}\textcolor{mycolor4}{$5.89^{\pm0.35}$}  &\cellcolor{gray!20}\textcolor{mycolor4}{$98.47^{\pm0.12}$}   &\cellcolor{gray!20}\textcolor{mycolor4}{$92.44^{\pm1.14}$} 
     \\  
     &DUL\textsuperscript{\dag} (ours)      &  \cellcolor{gray!20}$\underline{96.04}^{\pm0.03}$  &\cellcolor{gray!20}{$\underline{87.53}^{\pm0.70}$}                                                    &\cellcolor{gray!20}\textcolor{mycolor4}{${5.99}^{\pm0.25}$}  &\cellcolor{gray!20}\textcolor{mycolor4}{$98.28^{\pm0.11}$}   & \cellcolor{gray!20}\textcolor{mycolor4}{$\underline{98.40}^{\pm0.36}$} 
     \\ \midrule
\multirow{8}{*} {\shortstack{CIFAR-10\\ \\TIN-597}}         
       
     &Entropy       &   $\underline{95.94}^{\pm0.00}$    &\textcolor{mycolor3}{$80.51^{\pm0.68}$}                                                    &\textcolor{mycolor4}{$11.60^{\pm0.82}$}  &\textcolor{mycolor4}{$97.93^{\pm0.15}$}   & \textcolor{mycolor4}{${92.16}^{\pm0.50}$}                             \\  
    
     &EBM (pretrain)       &   \textcolor{mycolor3}{$95.38^{\pm0.13}$}     &\textcolor{mycolor3}{$83.67^{\pm1.41}$}                                                    &\textcolor{mycolor4}{$19.36^{\pm1.92}$}  &\textcolor{mycolor3}{$87.51^{\pm1.53}$}   &\textcolor{mycolor4}{$83.63^{\pm1.73}$}                              \\  
    
     &POEM       &   \textcolor{mycolor3}{$95.44^{\pm0.18}$}     &\textcolor{mycolor3}{$83.17^{\pm1.39}$}                                                    &\textcolor{mycolor4}{$24.34^{\pm2.48}$}  &\textcolor{mycolor3}{$86.83^{\pm1.13}$}   &\textcolor{mycolor4}{$\underline{94.25}^{\pm0.53}$}                 \\  &DPN      &   \textcolor{mycolor3}{$94.39^{\pm0.38}$} &\textcolor{mycolor3}{$79.23^{\pm2.95}$}                                                    &\textcolor{mycolor4}{$17.27^{\pm1.07}$}  &\textcolor{mycolor4}{$94.92^{\pm0.65}$}   &\textcolor{mycolor4}{$87.67^{\pm0.88}$} \\ &WOODS      &   \textcolor{mycolor3}{$95.57^{\pm0.64}$}     &\textcolor{mycolor3}{${83.12}^{\pm1.71}$}                                                    &\textcolor{mycolor4}{$\underline{7.58}^{\pm0.52}$}  &\textcolor{mycolor4}{$\textbf{98.29}^{\pm0.04}$}   & \textcolor{mycolor4}{${93.39}^{\pm0.39}$} 
                          \\  &SCONE      &  \textcolor{mycolor3}{$95.19^{\pm0.77}$} &\textcolor{mycolor3}{${84.68}^{\pm1.44}$}                                                    &\textcolor{mycolor4}{${8.02}^{\pm0.92}$}  &\textcolor{mycolor4}{$\underline{98.22}^{\pm0.08}$}   & \textcolor{mycolor4}{${93.08}^{\pm0.30}$}                           \\ 
     &DUL (ours)      &  \cellcolor{gray!20}$\textbf{96.06}^{\pm0.08}$  &\cellcolor{gray!20}\textcolor{mycolor4}{$\underline{87.93}^{\pm0.62}$}                                                    &\cellcolor{gray!20}\textcolor{mycolor4}{$\textbf{6.87}^{\pm0.82}$}  &\cellcolor{gray!20}\textcolor{mycolor4}{${98.21}^{\pm0.12}$}   & \cellcolor{gray!20}\textcolor{mycolor4}{$91.29^{\pm1.18}$}                             \\  
     &DUL\textsuperscript{\dag} (ours)      &  \cellcolor{gray!20}$\underline{95.94}^{\pm0.09}$  &\cellcolor{gray!20}\textcolor{mycolor4}{$\textbf{88.10}^{\pm0.27}$}                                                    &\cellcolor{gray!20}\textcolor{mycolor4}{${10.34}^{\pm0.34}$}  &\cellcolor{gray!20}\textcolor{mycolor4}{${97.67}^{\pm0.09}$}   & \cellcolor{gray!20}\textcolor{mycolor4}{$\textbf{98.59}^{\pm0.24}$}                               \\ \midrule
 \multirow{5}{*} { \shortstack{CIFAR-100\\ \\Only}}
     &MSP    &   $80.99^{\pm0.16}$   &   $55.95^{\pm1.38}$                                                 & $74.63^{\pm2.43}$ & $80.19^{\pm1.65}$  & $42.59^{\pm2.79}$                          \\  &EBM (pretrain)      &   $80.99^{\pm0.16}$   &   $55.95^{\pm1.38}$                                                   &$67.42^{\pm4.35}$  &$82.67^{\pm1.82}$   &$49.35^{\pm4.00}$                              \\ 
    
     &Maxlogits      & $80.99^{\pm0.16}$   &   $55.95^{\pm1.38}$                                                    &$69.32^{\pm3.97}$  &$82.30^{\pm1.79}$   & $47.60^{\pm3.68}$                             \\  &Mahalanobis      & $80.99^{\pm0.16}$   &   $55.95^{\pm1.38}$                                                     &$61.51^{\pm3.62}$  &$85.97^{\pm1.22}$   & $56.10^{\pm3.22}$                             \\ \midrule \multirow{8}{*} { \shortstack{CIFAR-100\\ \\ImageNet-RC}} &Entropy      &  \textcolor{mycolor3}{$80.21^{\pm0.09}$} &\textcolor{mycolor3}{$45.48^{\pm0.78}$}                                                    &\textcolor{mycolor4}{$22.29^{\pm1.32}$}  &\textcolor{mycolor4}{$95.33^{\pm0.28}$}   &\textcolor{mycolor4}{$82.34^{\pm1.11}$}                              \\   
    
     &EBM (finetune)      &   $80.53^{\pm0.22}$  &\textcolor{mycolor3}{$48.14^{\pm0.33}$}                                                    &\textcolor{mycolor4}{$13.47^{\pm0.43}$}  &\textcolor{mycolor4}{$\underline{96.78}^{\pm0.13}$}   &\textcolor{mycolor4}{${87.84}^{\pm0.86}$}                              \\  
    
     &POEM      &   \textcolor{mycolor3}{$78.15^{\pm0.18}$} &\textcolor{mycolor3}{$42.18^{\pm2.34}$}                                                    &\textcolor{mycolor4}{$\textbf{9.89}^{\pm0.36}$}  &\textcolor{mycolor4}{$\textbf{97.79}^{\pm0.12}$}   &\textcolor{mycolor4}{$\textbf{98.40}^{\pm0.08}$}                          \\  &DPN      &   \textcolor{mycolor3}{$78.90^{\pm0.25}$} &\textcolor{mycolor3}{$50.14^{\pm0.36}$}                              &\textcolor{mycolor4}{$18.36^{\pm0.82}$}  &\textcolor{mycolor4}{${95.42}^{\pm0.17}$}   &\textcolor{mycolor4}{~~$74.45^{\pm18.40}$}\\ &WOODS      &   $80.69^{\pm0.30}$     &\textcolor{mycolor3}{${54.38}^{\pm4.42}$}                                                    &\textcolor{mycolor4}{~~$38.15^{\pm12.91}$}  &\textcolor{mycolor4}{${92.01}^{\pm3.23}$}   & \textcolor{mycolor4}{${71.79}^{\pm7.98}$} 
                          \\       &SCONE      &  ${80.80}^{\pm0.30}$ &\textcolor{mycolor4}{$\textbf{56.73}^{\pm4.66}$}                                                    &\textcolor{mycolor4}{~~$47.60^{\pm14.73}$}  &\textcolor{mycolor4}{$89.61^{\pm3.75}$}   &\textcolor{mycolor4}{$65.29^{\pm9.66}$}                              \\ 
    
      &DUL (ours)    &    \cellcolor{gray!20}$\textbf{81.30}^{\pm0.19}$ &\cellcolor{gray!20}$\underline{56.27}^{\pm1.82}$                                                    &\cellcolor{gray!20}\textcolor{mycolor4}{${12.49}^{\pm0.22}$}  &\cellcolor{gray!20}\textcolor{mycolor4}{$95.24^{\pm0.09}$}   &\cellcolor{gray!20}\textcolor{mycolor4}{$86.72^{\pm0.76}$} 
      \\  
    &DUL\textsuperscript{\dag} (ours)      &  \cellcolor{gray!20}$\underline{81.22}^{\pm0.23}$  &\cellcolor{gray!20}{${56.07}^{\pm0.54}$}                                                    &\cellcolor{gray!20}\textcolor{mycolor4}{$\underline{11.75}^{\pm1.69}$}  &\cellcolor{gray!20}\textcolor{mycolor4}{$95.33^{\pm0.79}$}   & \cellcolor{gray!20}\textcolor{mycolor4}{$\underline{96.45}^{\pm0.42}$}\\  \midrule
\multirow{8}{*}{ \shortstack{CIFAR-100\\ \\TIN-597}}  &Entropy      &  \textcolor{mycolor3}{${80.15^{\pm0.17}}$} &\textcolor{mycolor3}{$46.25^{\pm1.42}$}                                                    &\textcolor{mycolor4}{$26.88^{\pm2.06}$}  &\textcolor{mycolor4}{${93.50^{\pm0.36}}$}   &\textcolor{mycolor4}{$79.81^{\pm1.31}$}                             \\  
    
    &EBM (finetune)      &   \textcolor{mycolor3}{$79.94^{\pm0.27}$}  &\textcolor{mycolor3}{$50.00^{\pm0.93}$}                                                    &\textcolor{mycolor4}{$26.87^{\pm1.15}$}  &\textcolor{mycolor4}{$91.68^{\pm0.45}$}   & \textcolor{mycolor4}{$80.08^{\pm0.76}$}                             \\  &POEM      &   \textcolor{mycolor3}{$78.68^{\pm0.13}$} &\textcolor{mycolor3}{$52.53^{\pm1.06}$}                                                    &\textcolor{mycolor4}{$32.71^{\pm0.96}$}  &\textcolor{mycolor4}{$91.30^{\pm0.68}$}   &\textcolor{mycolor4}{$\underline{94.65}^{\pm0.49}$}                          \\ &DPN      &   \textcolor{mycolor3}{$78.44^{\pm0.22}$} &\textcolor{mycolor3}{$47.67^{\pm0.28}$}                              &\textcolor{mycolor4}{${25.02}^{\pm2.19}$}  &\textcolor{mycolor4}{$\underline{93.55}^{\pm0.19}$}   &\textcolor{mycolor4}{${81.63}^{\pm1.27}$} \\ &WOODS      &   \textcolor{mycolor3}{$79.26^{\pm2.28}$}     &\textcolor{mycolor3}{${53.13}^{\pm2.97}$}                                                    &\textcolor{mycolor4}{${36.71}^{\pm7.92}$}  &\textcolor{mycolor4}{${92.15}^{\pm2.35}$}   & \textcolor{mycolor4}{${73.42}^{\pm4.72}$} 
                          \\  &SCONE      & \textcolor{mycolor3} {$79.53^{\pm1.82}$} &\textcolor{mycolor3}{${52.70}^{\pm0.96}$}                                                    &\textcolor{mycolor4}{$35.60^{\pm9.50}$}  &\textcolor{mycolor4}{$92.47^{\pm2.19}$}   &\textcolor{mycolor4}{$73.58^{\pm5.21}$}                              \\ 
    
      &DUL (ours)      &  \cellcolor{gray!20}$\textbf{80.85}^{\pm0.24}$ &\cellcolor{gray!20}$\underline{56.19}^{\pm1.53}$                                                    &\cellcolor{gray!20}\textcolor{mycolor4}{$\underline{23.32}^{\pm1.11}$}  &\cellcolor{gray!20}\textcolor{mycolor4}{$\textbf{94.48}^{\pm0.35}$}   &\cellcolor{gray!20}\textcolor{mycolor4}{$80.82^{\pm1.62}$}  
      \\  
     &DUL\textsuperscript{\dag} (ours)      &  \cellcolor{gray!20}$\underline{80.50}^{\pm0.25}$  &\cellcolor{gray!20}{$\textbf{56.22}^{\pm1.29}$}                                                    &\cellcolor{gray!20}\textcolor{mycolor4}{$\textbf{22.75}^{\pm0.88}$}  &\cellcolor{gray!20}\textcolor{mycolor4}{$90.88^{\pm0.28}$}   & \cellcolor{gray!20}\textcolor{mycolor4}{$\textbf{96.33}^{\pm0.09}$} \\
    
       \bottomrule
\end{tabular}}
\end{center}
\end{table*}

\begin{table*}[!htbp]
\vskip 0.15in
\begin{center}
\caption{OOD detection and generalization performance comparison with standard variance. Substantially \textcolor{mycolor4}{improvement} and \textcolor{mycolor3}{degradation} ($\geq 0.5$) compare to baseline method w.r.t. MSP are highlighted in blue or red respectively. The \textbf{best} and \underline{second best} results are in bold or underlined. Similar with CIFAR experiments, DUL establishes strong OOD detection performance (always the best or second best) without degraded generalization i.e., the entire row is either blue or black.}
\vskip 0.15in
\label{tab:full-results-with-std on imagenet}
\center
\resizebox{1.0\textwidth}{!}{
\setlength{\tabcolsep}{3.9mm}
\begin{tabular}{c|c|cc|ccc}
\toprule \multirow{2}{*} {$\mathcal{P}^{\rm ID}/\mathcal{P}^{\rm SEM}_{\rm train}$}   &\multirow{2}{*} {Method}    &  \multicolumn{2}{c|}{ID/OOD Generalization}                                                    &\multicolumn{3}{c}{OOD Detection}                              \\
 &   &\text{ID-Acc. $\uparrow$} & \text{OOD-Acc. $\uparrow$} & \text{FPR $\downarrow$}& \text{AUROC $\uparrow$}  &\text{AUPR $\uparrow$} \\ \midrule 
\multirow{8}{*} {\shortstack{ImageNet-200\\ \\ImageNet-800}}         &MSP        &  $\underline{85.15}^{{\pm0.33}}$   &  $74.84^{{\pm0.47}}$                                                  &$58.23^{{\pm1.54}}$  & $86.98^{{\pm0.24}}$  & $82.27^{{\pm0.32}}$                          \\  &EBM (pretrain)      &  $\underline{85.15}^{{\pm0.33}}$&  $74.84^{{\pm0.47}}$                                                  & $51.94^{{\pm0.82}}$ & $88.18^{{\pm0.11}}$   &$84.75^{{\pm0.08}}$                              \\ &Maxlogits      &  $\underline{85.15}^{{\pm0.33}}$&  $74.84^{{\pm0.47}}$                                                  & $\underline{51.62}^{{\pm0.20}}$ & $88.30^{{\pm0.09}}$   &$84.71^{{\pm0.07}}$                              \\ \cline{2-7} 
    
     &Entropy      &   {${84.92}^{\pm0.30}$}    &${74.75}^{{\pm0.52}}$                                                    &\textcolor{mycolor4}{${53.62}^{{\pm0.76}}$}  &\textcolor{mycolor4}{$\underline{89.05}^{{\pm0.01}}$}   &\textcolor{mycolor4}{$\underline{85.02}^{{\pm0.08}}$}                              \\  
    
     &EBM (finetune)      &   \textcolor{mycolor3}{$84.14^{{\pm0.11}}$}     &\textcolor{mycolor3}{$73.31^{{\pm0.57}}$}                                                    &\textcolor{mycolor3}{$59.73^{{\pm0.83}}$}  &\textcolor{mycolor4}{$87.54^{{\pm0.03}}$}     &\textcolor{mycolor4}{${82.81}^{{\pm0.16}}$}                                      \\  
    
     &DPN      &   $84.87^{{\pm0.30}}$     &$74.40^{{\pm0.90}}$                                                    &\textcolor{mycolor3}{$63.84^{{\pm0.70}}$}  &$87.18^{{\pm0.18}}$   &\textcolor{mycolor3}{$80.69^{{\pm0.35}}$}  \\  
    
     &WOODS      &   ${84.99}^{{\pm0.62}}$     &$\underline{74.98}^{{\pm0.46}}$                                                    &\textcolor{mycolor4}{${51.71}^{{\pm2.84}}$}  &\textcolor{mycolor4}{$88.30^{{\pm0.56}}$}   &\textcolor{mycolor4}{${84.80}^{{\pm0.98}}$}\\  &SCONE      &  ${84.93}^{{\pm0.71}}$ &${74.91}^{{\pm0.49}}$                                                    &\textcolor{mycolor4}{${52.52}^{{\pm3.54}}$}  &\textcolor{mycolor4}{${88.19}^{{\pm0.41}}$}   &\textcolor{mycolor4}{${84.50}^{{\pm1.08}}$}                              \\
     &DUL (ours)      & \cellcolor{gray!20}\textcolor{mycolor4}{$\textbf{85.65}^{\pm0.07}$}  &\cellcolor{gray!20}\textcolor{mycolor4}{$\textbf{75.59}^{\pm0.12}$}                                                    &\cellcolor{gray!20}\textcolor{mycolor4}{$\textbf{49.14}^{\pm0.13}$}  &\cellcolor{gray!20}\textcolor{mycolor4}{$\textbf{89.27}^{\pm0.03}$}   &\cellcolor{gray!20}\textcolor{mycolor4}{$\textbf{85.62}^{\pm0.03}$}                              \\  
       \bottomrule
\end{tabular}}
\end{center}
\end{table*}

\begin{table}[htbp]
\centering
\caption{OOD detection results of DUL on each individual OOD detection test dataset.}
\label{tab:ood detection performance on each OOD dataset}
\resizebox{\textwidth}{!}{
\begin{tabular}{c|c|cc|cc|cc|cc|cc|cc}
\toprule
  
 \multirow{2}{*}{$\mathcal{P}^{\rm ID}/\mathcal{P}^{\rm SEM}_{\rm train}$}& & \multicolumn{2}{c}{LSUN-crop} & \multicolumn{2}{c}{Places365} & \multicolumn{2}{c}{LSUN-resize} & \multicolumn{2}{c}{iSUN} & \multicolumn{2}{c}{Texture} & \multicolumn{2}{c}{SVHN} \\ & Method & FPR↓ & AUROC↑ & FPR↓ & AUROC↑ & FPR↓ & AUROC↑ & FPR↓ & AUROC↑ & FPR↓ & AUROC↑ & FPR↓ & AUROC↑  \\
\midrule
\multirow{2}{*}{\shortstack{CIFAR-10\\ \\ImageNet-RC}} &DUL &6.75  &98.75 &15.55  &96.34  &0.00 &99.63  &0.00  &99.57  &3.20  &98.89  &8.35 &98.24 \\&DUL\textsuperscript{\dag}&12.79  &98.08 &14.99  &96.11  &0.00 &99.49  &0.00  &99.45  &1.72  &98.79  &4.84 &98.47 \\\midrule
\multirow{2}{*}{\shortstack{CIFAR-10\\ \\TIN-597}} &DUL&0.90  &99.47 &22.40  &95.21  &0.10 &99.52  &0.30  &99.45  &8.00  &97.94  &5.90 &98.52 \\&DUL\textsuperscript{\dag} &3.48  &99.25 &32.81  &91.48  &0.00 &99.78  &0.00  &99.78  &13.97  &97.15  &9.69 &98.28 \\\midrule
\multirow{2}{*}{\shortstack{CIFAR-100\\ \\ImageNet-RC}}&DUL &44.65  &83.73 &21.05  &79.30  &0.00 &99.63  &0.00  &99.61  &3.05  &98.53  &5.00 &98.15 \\&DUL\textsuperscript{\dag} &47.75  &81.29 &14.92  &95.08  &0.00 &99.58  &0.00  &99.46  &1.70  &98.70  &4.44 &98.02 \\\midrule
\multirow{2}{*}{\shortstack{CIFAR-100\\ \\TIN-597}}&DUL &7.05  &98.69 &65.35  &83.60  &3.80 &99.03  &3.50  &98.98  &34.45  &92.21  &21.15 &95.81 \\&DUL\textsuperscript{\dag} &6.25  &98.55 &78.95  &64.09  &0.02 &99.84  &0.00  &99.83  &39.79  &85.02  &5.61 &98.51 \\
\bottomrule
\end{tabular}
}
\end{table}

\begin{table}[htbp]
\centering
\caption{OOD detection results of DUL on each OOD detection test dataset.}
\label{tab:ood detection performance on each OOD dataset imagenet}
\resizebox{\textwidth}{!}{
\begin{tabular}{cccccccccccc}
\toprule
\multirow{2}{*}{$\mathcal{P}^{\rm ID}/\mathcal{P}^{\rm SEM}_{\rm train}$} &\multirow{2}{*}{Method} & \multicolumn{2}{c}{OpenImage-O} & \multicolumn{2}{c}{SSB-hard} & \multicolumn{2}{c}{Textures} & \multicolumn{2}{c}{iNaturalist} & \multicolumn{2}{c}{NINCO}  \\ 
\cmidrule(r){3-4} \cmidrule(r){5-6} \cmidrule(r){7-8} \cmidrule(r){9-10} \cmidrule(r){11-12}  
 & & FPR↓ & AUROC↑ & FPR↓ & AUROC↑ & FPR↓ & AUROC↑ & FPR↓ & AUROC↑ & FPR↓ & AUROC↑  \\
\midrule
\multirow{1}{*}{\shortstack{ImageNet-200/800}} &DUL &49.68  &91.31 &72.40  &80.60  &30.76 &92.98  &33.21  &94.76  &59.92  &86.61   \\
\bottomrule
\end{tabular}
}
\end{table}

\begin{table}[htbp]
\centering
\caption{Classification error rate comparison on CIFAR10-C. ID dataset is CIFAR10.}
\label{tab:ood generalization cifar10-C}
\resizebox{\textwidth}{!}{
\begin{tabular}{c|c|ccc|cccc|cccc|cccc|c}
\toprule
  & & \multicolumn{3}{c}{Noise} & \multicolumn{4}{c}{Blur} & \multicolumn{4}{c}{Weather} & \multicolumn{4}{c}{Digital} & \multicolumn{1}{c}{} \\  
 Method &$\mathcal{P}^{\rm SEM}_{\rm train}$ &  Gauss. & Shot & Impul. & Defoc. & Glass & Motion & Zoom & Snow & Frost & Fog & Brit. & Contr. &Elastic &Pixel &JPEG &Avg.  \\
\midrule
MSP &\multirow{4}{*}{None} & 53.35 &39.98 & 44.07 & 15.61  &42.35  & 19.33 & 19.83 & 14.39 & 18.02 &9.61 &5.09& 18.19&13.45 &22.71&18.61&23.64\\EBM (pretrain) & & 53.35 &39.98 & 44.07 & 15.61  &42.35  & 19.33 & 19.83 & 14.39 & 18.02 &9.61 &5.09& 18.19&13.45 &22.71&18.61&23.64 \\Maxlogits & & 53.35 &39.98 & 44.07 & 15.61  &42.35  & 19.33 & 19.83 & 14.39 & 18.02 &9.61 &5.09& 18.19&13.45 &22.71&18.61&23.64  \\Mahalanobis  & & 53.35 &39.98 & 44.07 & 15.61  &42.35  & 19.33 & 19.83 & 14.39 & 18.02 &9.61 &5.09& 18.19&13.45 &22.71&18.61&23.64  \\ \midrule Entropy &\multirow{7}{*}{ImageNet-RC} &  67.20&53.99&70.93&17.09&81.99&20.00&22.17&22.41&31.78&11.43&5.82&16.99&14.28&28.92&21.37&32.42 \\EBM (finetune) & &62.49&50.00&73.01&21.01&73.93&20.06&25.80&21.02&27.00&11.55&5.60&15.08&15.15&31.46&22.69&31.72 \\DPN & &50.19&38.80&59.81&16.98&52.81&19.68&22.24&18.25&20.11&11.10&5.74&18.01&14.41&24.30&18.76&26.08 \\POEM & &47.52&38.16&63.39&22.64&67.37&24.12&28.43&23.44&27.33&14.30&7.65&19.02&18.37&31.87&22.23&30.40\\WOODS & &62.36&49.09&60.77&15.56&74.37&19.05&20.16&18.79&25.41&9.84&5.45&16.76&14.16&24.07&20.45&29.09\\SCONE& &63.29&50.01&61.96&15.61&77.16&18.65&20.24&19.77&26.43&9.82&5.54&16.88&14.09&24.46&20.64&29.64 \\DUL (Ours) & &53.46&40.11&43.71&15.62&42.77&19.49&19.89&14.27&18.07&9.82&5.06&18.49&13.56&22.65&18.69&23.71 \\\midrule Entropy &\multirow{7}{*}{TIN-597} &65.68&52.08&56.49&18.68&51.96&22.73&24.43&18.57&24.64&10.56&5.51&16.38&15.88&25.19&41.25&30.00\\EBM (finetune) & &58.95&44.89&52.91&18.30&47.25&23.99&24.33&17.70&22.18&12.23&6.07&19.78&17.11&26.85&51.31&29.59 \\DPN & &48.07&38.05&42.56&22.54&47.39&27.29&30.41&20.46&26.17&14.53&7.22&24.54&18.33&27.40&26.23&28.08 \\POEM & &51.06&39.56&47.20&16.08&48.97&20.23&21.40&17.36&22.15&11.29&5.94&18.08&14.87&22.08&31.74&25.87 \\WOODS & &60.01&46.74&51.26&18.04&49.83&20.95&23.17&17.34&22.66&10.02&5.69&16.02&16.32&23.55&25.92&27.17 \\SCONE & &58.20&44.92&50.36&17.14&48.47&20.64&22.80&17.05&21.15&9.87&6.07&17.24&16.17&24.89&25.53&26.70 \\DUL (Ours)& &54.56&40.84&44.18&15.67&43.61&19.93&19.74&14.49&18.34&9.84&5.18&19.24&13.95&22.91&18.66&24.08 \\
\bottomrule
\end{tabular}
}
\end{table}

\begin{table}[htbp]
\centering
\caption{Classification error rate comparison on CIFAR100-C. ID dataset is CIFAR100.}
\label{tab:ood generalization cifar100-C}
\resizebox{\textwidth}{!}{
\begin{tabular}{c|c|ccc|cccc|cccc|cccc|c}
\toprule
  & & \multicolumn{3}{c}{Noise} & \multicolumn{4}{c}{Blur} & \multicolumn{4}{c}{Weather} & \multicolumn{4}{c}{Digital} & \multicolumn{1}{c}{} \\  
 Method &$\mathcal{P}^{\rm SEM}_{\rm train}$ &  Gauss. & Shot & Impul. & Defoc. & Glass & Motion & Zoom & Snow & Frost & Fog & Brit. & Contr. &Elastic &Pixel &JPEG &Avg.  \\
\midrule
MSP &\multirow{4}{*}{None} &76.98&67.70&74.51&36.08&78.49&40.69&41.39&40.52&46.66&31.44&22.59&40.29&35.31&42.83&44.93&48.03 \\EBM (pretrain) & &76.98&67.70&74.51&36.08&78.49&40.69&41.39&40.52&46.66&31.44&22.59&40.29&35.31&42.83&44.93&48.03 \\Maxlogits &&76.98&67.70&74.51&36.08&78.49&40.69&41.39&40.52&46.66&31.44&22.59&40.29&35.31&42.83&44.93&48.03 \\Mahalanobis & &76.98&67.70&74.51&36.08&78.49&40.69&41.39&40.52&46.66&31.44&22.59&40.29&35.31&42.83&44.93&48.03 \\ \midrule Entropy &\multirow{7}{*}{ImageNet-RC} &81.69&72.89&86.99&38.19&88.94&42.20&44.62&45.50&53.90&33.14&24.40&39.77&37.19&47.62&48.82&52.39 \\EBM (finetune) & &82.33&73.51&90.58&39.20&90.71&41.43&44.40&47.29&54.88&34.26&24.72&39.15&37.65&51.47&49.29&53.39\\DPN & &77.60&68.59&83.59&39.74&86.77&43.33&44.91&47.01&53.69&36.37&26.11&43.45&37.56&47.48&44.47&52.04 \\POEM & &83.65&76.63&88.02&44.09&90.05&44.29&46.44&52.80&60.84&39.49&28.90&43.35&42.49&56.06&54.50&56.77 \\WOODS & &77.51&68.48&77.47&36.38&80.62&40.92&41.93&41.27&47.08&31.36&23.30&39.74&36.21&42.87&46.16&48.75 \\SCONE & &76.43&67.30&75.32&36.23&77.00&40.92&41.59&40.22&45.53&31.35&23.09&40.10&35.87&42.42&45.42&47.92\\DUL (Ours) & &77.23&67.95&75.13&35.83&78.52&39.74&40.64&39.76&46.13&30.96&22.43&39.17&34.81&43.03&44.64&47.73 \\\midrule Entropy &\multirow{7}{*}{TIN-597} &83.34&75.44&79.77&38.02&82.74&42.22&44.17&44.52&53.36&32.80&23.67&38.00&37.85&44.55&62.91&52.22\\EBM (finetune) & &81.24&72.78&78.14&37.55&80.00&42.94&43.64&43.70&51.11&33.52&23.92&40.19&38.13&45.05&73.27&52.34\\DPN & &81.81&72.87&79.09&39.61&81.25&44.66&46.21&45.42&52.40&34.28&25.64&40.89&39.15&47.66&76.40&53.82\\POEM & &78.88&70.32&74.68&38.59&77.95&43.41&44.27&43.00&50.22&34.64&25.37&43.09&37.28&45.43&81.55&52.58\\WOODS & &81.14&72.39&76.49&38.50&79.00&43.23&44.59&42.37&49.21&33.03&24.32&40.23&39.01&41.19&47.80&50.17 \\SCONE & &80.90&72.30&77.19&37.91&78.70&42.36&43.99&42.81&49.49&32.19&24.00&39.26&37.96&41.80&48.34&49.95\\DUL (Ours) & &77.01&67.67&74.16&36.16&78.35&40.79&41.32&40.17&46.33&31.42&22.87&40.79&35.34&42.74&45.05&48.01 \\
\bottomrule
\end{tabular}
}
\end{table}

\begin{table}[htbp]
\centering
\caption{Classification error rate comparison on ImageNet-C. Here we test compared methods on a subset of the original ImageNet-C consisting of 200 classes. ID dataset is ImageNet-200.}
\label{tab:ood generalization ImageNet-C}
\resizebox{\textwidth}{!}{
\begin{tabular}{c|c|ccc|cccc|cccc|cccc|c}
\toprule
  & & \multicolumn{3}{c}{Noise} & \multicolumn{4}{c}{Blur} & \multicolumn{4}{c}{Weather} & \multicolumn{4}{c}{Digital} & \multicolumn{1}{c}{} \\  
 Method &$\mathcal{P}^{\rm SEM}_{\rm train}$ &  Gauss. & Shot & Impul. & Defoc. & Glass & Motion & Zoom & Snow & Frost & Fog & Brit. & Contr. &Elastic &Pixel &JPEG &Avg.  \\
\midrule
MSP &\multirow{4}{*}{None}&52.2&70.7&715&56.0&55.5&52.6&54.1&67.3&67.0&63.5&56.5&53.3&46.3&50.3&51.9&57.9 \\EBM (pretrain)& &52.2&70.7&715&56.0&55.5&52.6&54.1&67.3&67.0&63.5&56.5&53.3&46.3&50.3&51.9&57.9\\Maxlogits & &52.2&70.7&715&56.0&55.5&52.6&54.1&67.3&67.0&63.5&56.5&53.3&46.3&50.3&51.9&57.9 \\ \midrule Entropy &\multirow{6}{*}{ImageNet-800} &51.3&71.2&71.7&54.4&54.6&51.8&53.2&66.9&66.2&62.7&55.8&51.4&45.1&49.7&51.1&57.1 \\EBM (finetune) & &52.9&72.2&72.8&56.0&56.2&53.5&54.1&67.8&67.4&64.0&57.0&52.2&46.5&51.4&52.9&58.5 \\DPN & &51.9&69.2&69.6&56.5&55.0&52.0&53.1&65.5&65.3&62.3&55.5&54.5&46.0&49.7&51.4&57.2 \\WOODS & &51.4&69.4&70.0&55.2&54.8&51.9&52.9&66.2&66.0&62.5&55.6&52.6&45.6&49.7&51.3&57.0 \\SCONE & &51.6&69.4&70.0&55.4&55.0&52.1&53.1&66.3&66.0&62.6&55.7&53.0&45.8&49.9&51.4&57.1 \\DUL (Ours) & &51.1&69.1&70.5&55.1&54.5&51.6&52.4&66.2&65.9&62.6&55.7&53.0&45.4&49.4&50.9&56.9 \\
\bottomrule
\end{tabular}
}
\end{table}

\begin{table}[!htbp]
\centering
\caption{Comprehensive comparison involving 15 different types of corruption from commonly-used domain adaption benchmark [52]. Substantial ($\geq0.5$) \textcolor{mycolor4}{improvement} and \textcolor{mycolor3}{degradation} compared to the baseline MSP [6] are highlighted in \textcolor{mycolor4}{blue} or \textcolor{mycolor3}{red} respectively. DUL is the only method that achieves SOTA OOD detection performance without sacrificing generalization i.e., the value of the entire row is almost black or blue. The \textbf{best} or \underline{second best} results are highlighted in bold or underlined. MD is the shorthand of Mahalanobis.}
\vskip 4pt
\label{tab:ood generalization cifar10-C}
\resizebox{\textwidth}{!}{
\setlength{\tabcolsep}{0.75mm}
\begin{tabular}{c|c|c|ccccccccccccccc|c|cc}
\toprule
 \multicolumn{1}{c}{} &\multicolumn{1}{c}{} &\multicolumn{1}{c}{} & \multicolumn{16}{c}{OOD generalization (Error rate $\downarrow$)} & \multicolumn{2}{c}{OOD detection} \\  
  Method &$\mathcal{P}^{\rm ID}_{\rm train}$  &$\mathcal{P}^{\rm SEM}_{\rm train}$ &  Gauss. & Shot & Impul. & Defoc. & Glass & Motion & Zoom & Snow & Frost & Fog & Brit. & Contr. &Elast. &Pixel &JPEG &Avg. &FPR$\downarrow$ &AUC$\uparrow$\\
\midrule
MSP &\multirow{18}{*}{\rotatebox{90}{CIFAR-10}} &\multirow{4}{*}{\rotatebox{90}{None}} & 77.0 & 67.7 & 74.5 & 36.1 & 78.5 & 40.7 & 41.4 & 40.5 & 46.7 & 31.4 & 22.6 & 40.3 & 35.3 & 42.8 & 44.9 & 48.0 &42.0 &89.3\\EBM  & & & 77.0 & 67.7 & 74.5 & 36.1 & 78.5 & 40.7 & 41.4 & 40.5 & 46.7 & 31.4 & 22.6 & 40.3 & 35.3 & 42.8 & 44.9 & 48.0&32.5 &89.3\\Maxlogits & & &77.0 & 67.7 & 74.5 & 36.1 & 78.5 & 40.7 & 41.4 & 40.5 & 46.7 & 31.4 & 22.6 & 40.3 & 35.3 & 42.8 & 44.9 & 48.0 &32.9 &89.3 \\ MD & & & 77.0 & 67.7 & 74.5 & 36.1 & 78.5 & 40.7 & 41.4 & 40.5 & 46.7 & 31.4 & 22.6 & 40.3 & 35.3 & 42.8 & 44.9 & 48.0 &32.5 &93.9 \\
\cline{1-1} \cline{3-21} Entropy& &\multirow{7}{*}{\rotatebox{90}{ImageNet-RC}} &\textcolor{mycolor3}{81.7} & \textcolor{mycolor3}{72.9} & \textcolor{mycolor3}{87.0} & \textcolor{mycolor3}{38.2} & \textcolor{mycolor3}{88.9} & \textcolor{mycolor3}{42.2} & \textcolor{mycolor3}{44.6} & \textcolor{mycolor3}{45.5} & \textcolor{mycolor3}{53.9} & \textcolor{mycolor3}{33.1} & \textcolor{mycolor3}{24.4} & \textcolor{mycolor4}{39.8} & \textcolor{mycolor3}{37.2} & \textcolor{mycolor3}{47.6} & \textcolor{mycolor3}{48.8} & \textcolor{mycolor3}{52.4} &\textcolor{mycolor4}{6.6} &\textcolor{mycolor4}{\underline{98.7}}\\EBM (FT) & & &\textcolor{mycolor3}{82.3} & \textcolor{mycolor3}{73.5} & \textcolor{mycolor3}{90.6} & \textcolor{mycolor3}{39.2} & \textcolor{mycolor3}{90.7} & \textcolor{mycolor3}{41.4} & \textcolor{mycolor3}{44.4} & \textcolor{mycolor3}{47.3} & \textcolor{mycolor3}{54.9} & \textcolor{mycolor3}{34.3} & \textcolor{mycolor3}{24.7} & \textcolor{mycolor4}{\textbf{39.2}} & \textcolor{mycolor3}{37.7} & \textcolor{mycolor3}{51.5} & \textcolor{mycolor3}{49.3} & \textcolor{mycolor3}{53.4} &\textcolor{mycolor4}{\underline{3.6}} &\textcolor{mycolor4}{98.4}\\DPN & & &\textcolor{mycolor3}{77.6} & \textcolor{mycolor3}{68.6} & \textcolor{mycolor3}{83.6} & \textcolor{mycolor3}{39.7} & \textcolor{mycolor3}{86.8} & \textcolor{mycolor3}{43.3} & \textcolor{mycolor3}{44.9} & \textcolor{mycolor3}{47.0} & \textcolor{mycolor3}{53.7} & \textcolor{mycolor3}{36.4} & \textcolor{mycolor3}{26.1} & \textcolor{mycolor3}{43.5} & \textcolor{mycolor3}{37.6} & \textcolor{mycolor3}{47.5} & \textbf{44.5} & \textcolor{mycolor3}{52.0} &\textcolor{mycolor4}{4.3} &\textcolor{mycolor4}{98.5}\\ POEM & & & \textcolor{mycolor3}{83.7} & \textcolor{mycolor3}{76.6} & \textcolor{mycolor3}{88.0} & \textcolor{mycolor3}{44.1} & \textcolor{mycolor3}{90.1} & \textcolor{mycolor3}{44.3} & \textcolor{mycolor3}{46.4} & \textcolor{mycolor3}{52.8} & \textcolor{mycolor3}{60.8} & \textcolor{mycolor3}{39.5} & \textcolor{mycolor3}{28.9} & \textcolor{mycolor3}{43.4} & \textcolor{mycolor3}{42.5} & \textcolor{mycolor3}{56.1} & \textcolor{mycolor3}{54.5} & \textcolor{mycolor3}{56.8} &\textcolor{mycolor4}{\textbf{3.3}} &\textcolor{mycolor4}{\textbf{99.0}}\\WOODS & & &\textcolor{mycolor3}{77.5} & \textcolor{mycolor3}{68.5} & \textcolor{mycolor3}{77.5} & 36.4 & \textcolor{mycolor3}{80.6} & \underline{40.9} & \textcolor{mycolor3}{41.9} & \textcolor{mycolor3}{41.3} & 47.1 & \underline{31.4} & \textcolor{mycolor3}{23.3} & \textcolor{mycolor4}{\underline{39.7}} & \textcolor{mycolor3}{36.2} & \underline{42.9} & \textcolor{mycolor3}{46.2} & \textcolor{mycolor3}{48.8} &\textcolor{mycolor4}{7.1} &\textcolor{mycolor4}{98.5}\\SCONE& & & \textcolor{mycolor4}{\textbf{76.4}} & \textbf{67.3} & \textcolor{mycolor3}{\underline{75.3}} & \underline{36.2} & \textcolor{mycolor4}{\textbf{77.0}} & \underline{40.9} & \underline{41.6} & \underline{40.2} & \textcolor{mycolor4}{\textbf{45.5}} & \underline{31.4} & \textcolor{mycolor3}{\underline{23.1}} & 40.1 & \textcolor{mycolor3}{\underline{35.9}} & \textbf{42.4} & \textcolor{mycolor3}{45.4} & \underline{47.9} &\textcolor{mycolor4}{7.0} &\textcolor{mycolor4}{98.5}\\DUL (Ours) & & &\cellcolor{gray!20}\underline{77.2} & \cellcolor{gray!20}\underline{68.0} & \cellcolor{gray!20}\textcolor{mycolor3}{\textbf{75.1}} & \cellcolor{gray!20}\textbf{35.8} & \cellcolor{gray!20}\underline{78.5} & \cellcolor{gray!20}\textcolor{mycolor4}{\textbf{39.7}} & \cellcolor{gray!20}\textcolor{mycolor4}{\textbf{40.6}} & \cellcolor{gray!20}\textcolor{mycolor4}{\textbf{39.8}} & \cellcolor{gray!20}\textcolor{mycolor4}{\underline{46.1}} & \cellcolor{gray!20}\textbf{31.0} & \cellcolor{gray!20}\textbf{22.4} & \cellcolor{gray!20}\textcolor{mycolor4}{\textbf{39.2}} & \cellcolor{gray!20}\textcolor{mycolor4}{\textbf{34.8}} & \cellcolor{gray!20}43.0 & \cellcolor{gray!20}\underline{44.6} & \cellcolor{gray!20}\textbf{47.7} &\cellcolor{gray!20}\textcolor{mycolor4}{5.9} &\cellcolor{gray!20}\textcolor{mycolor4}{98.5}\\
\cline{1-1} \cline{3-21} Entropy & &\multirow{7}{*}{\rotatebox{90}{TIN-597}} & \textcolor{mycolor3}{83.3} & \textcolor{mycolor3}{75.4} & \textcolor{mycolor3}{79.8} & \textcolor{mycolor3}{38.0} & \textcolor{mycolor3}{82.7} & \textcolor{mycolor3}{\underline{42.2}} & \textcolor{mycolor3}{44.2} & \textcolor{mycolor3}{44.5} & \textcolor{mycolor3}{53.4} & \textcolor{mycolor3}{32.8} & \textcolor{mycolor3}{\underline{23.7}} & \textcolor{mycolor4}{\textbf{38.0}} & \textcolor{mycolor3}{37.9} & \textcolor{mycolor3}{44.6} & \textcolor{mycolor3}{62.9} & \textcolor{mycolor3}{52.2} &\textcolor{mycolor4}{11.6} &\textcolor{mycolor4}{97.9}\\EBM (FT) & & &\textcolor{mycolor3}{81.2} & \textcolor{mycolor3}{72.8} & \textcolor{mycolor3}{78.1} & \textcolor{mycolor3}{\underline{37.6}} & \textcolor{mycolor3}{80.0} & \textcolor{mycolor3}{42.9} & \textcolor{mycolor3}{43.6} & \textcolor{mycolor3}{43.7} & \textcolor{mycolor3}{51.1} & \textcolor{mycolor3}{33.5} & \textcolor{mycolor3}{23.9} & 40.2 & \textcolor{mycolor3}{38.1} & \textcolor{mycolor3}{45.1} & \textcolor{mycolor3}{73.3} & \textcolor{mycolor3}{52.3}&\textcolor{mycolor4}{19.4} &\textcolor{mycolor3}{87.5}\\DPN & & &\textcolor{mycolor3}{81.8} & \textcolor{mycolor3}{72.9} & \textcolor{mycolor3}{79.1} & \textcolor{mycolor3}{39.6} & \textcolor{mycolor3}{81.3} & \textcolor{mycolor3}{44.9} & \textcolor{mycolor3}{46.2} & \textcolor{mycolor3}{45.4} & \textcolor{mycolor3}{52.4} & \textcolor{mycolor3}{34.3} & \textcolor{mycolor3}{25.6} & \textcolor{mycolor3}{40.9} & \textcolor{mycolor3}{39.2} & \textcolor{mycolor3}{47.7} & \textcolor{mycolor3}{76.4} & \textcolor{mycolor3}{53.8}&\textcolor{mycolor4}{17.3} &\textcolor{mycolor4}{94.9}\\POEM & & &\textcolor{mycolor3}{\underline{78.9}} & \textcolor{mycolor3}{\underline{70.3}} & \underline{74.7} & \textcolor{mycolor3}{38.6} & \textcolor{mycolor4}{\textbf{78.0}} & \textcolor{mycolor3}{43.4} & \textcolor{mycolor3}{44.3} & \textcolor{mycolor3}{43.0} & \textcolor{mycolor3}{50.2} & \textcolor{mycolor3}{34.6} & \textcolor{mycolor3}{25.4} & \textcolor{mycolor3}{43.1} & \textcolor{mycolor3}{\underline{37.3}} & \textcolor{mycolor3}{45.4} & \textcolor{mycolor3}{81.6} & \textcolor{mycolor3}{52.6} &\textcolor{mycolor4}{34.3} &\textcolor{mycolor3}{86.8}\\WOODS & & & \textcolor{mycolor3}{81.1} & \textcolor{mycolor3}{72.4} & \textcolor{mycolor3}{76.5} & \textcolor{mycolor3}{38.6} & \textcolor{mycolor3}{79.0} & \textcolor{mycolor3}{43.2} & \textcolor{mycolor3}{44.6} & \textcolor{mycolor3}{\underline{42.4}} & \textcolor{mycolor3}{\underline{49.2}} & \textcolor{mycolor3}{33.0} & \textcolor{mycolor3}{24.3} & 40.2 & \textcolor{mycolor3}{39.0} & \textcolor{mycolor4}{\textbf{41.2}} & \textcolor{mycolor3}{\underline{47.8}} & \textcolor{mycolor3}{50.2} &\textcolor{mycolor4}{\underline{7.6}} &\textcolor{mycolor4}{\textbf{98.3}}\\SCONE & & & \textcolor{mycolor3}{80.9} & \textcolor{mycolor3}{72.3} & \textcolor{mycolor3}{77.2} & \textcolor{mycolor3}{37.9} & 78.7 & \textcolor{mycolor3}{42.4} & \textcolor{mycolor3}{\underline{43.4}} & \textcolor{mycolor3}{42.8} & \textcolor{mycolor3}{49.5} & \textcolor{mycolor3}{\underline{32.2}} & \textcolor{mycolor3}{24.0} & \textcolor{mycolor4}{\underline{39.3}} & \textcolor{mycolor3}{38.0} & \textcolor{mycolor4}{\underline{41.8}} & \textcolor{mycolor3}{48.3} & \textcolor{mycolor3}{\underline{50.0}} &\textcolor{mycolor4}{8.0} &\textcolor{mycolor4}{\underline{98.2}}\\DUL (Ours)& & &\cellcolor{gray!20}\textbf{77.0} & \cellcolor{gray!20}\textbf{67.7} & \cellcolor{gray!20}\textbf{74.2} & \cellcolor{gray!20}\textbf{36.2} & \cellcolor{gray!20}\underline{78.4} & \cellcolor{gray!20}\textbf{40.8} & \cellcolor{gray!20}\textbf{41.3} & \cellcolor{gray!20}\textbf{40.2} & \cellcolor{gray!20}\textbf{46.3} & \cellcolor{gray!20}\textbf{31.4} & \cellcolor{gray!20}\textbf{22.9} & \cellcolor{gray!20}\textcolor{mycolor3}{40.8} & \cellcolor{gray!20}\textbf{35.3} & \cellcolor{gray!20}42.7 & \cellcolor{gray!20}\textbf{45.1} & \cellcolor{gray!20}\textbf{48.0}&\cellcolor{gray!20}\textcolor{mycolor4}{\textbf{6.9}} &\cellcolor{gray!20}\textcolor{mycolor4}{\underline{98.2}}\\
\midrule \midrule

MSP &\multirow{9}{*}{\rotatebox{90}{ImageNet-200}} &\multirow{3}{*}{\rotatebox{90}{None}} & 52.2 & 70.7 & 71.5 & 56.0  &55.5  & 52.6 & 54.1 & 67.3 & 67.0 & 63.5 &56.5& 53.3&46.3 &50.3&51.9&57.9 &58.2 &82.3\\EBM  & & & 52.2 & 70.7 & 71.5 & 56.0  &55.5  & 52.6 & 54.1 & 67.3 & 67.0 & 63.5 &56.5& 53.3&46.3 &50.3&51.9&57.9&32.5 &89.3\\Maxlogits & & &52.2 & 70.7 & 71.5 & 56.0  &55.5  & 52.6 & 54.1 & 67.3 & 67.0 & 63.5 &56.5& 53.3&46.3 &50.3&51.9&57.9 &51.9 &88.2 \\
\cline{1-1} \cline{3-21} Entropy& &\multirow{6}{*}{\rotatebox{90}{ImageNet-800}} & \textcolor{mycolor4}{\underline{51.3}}&\textcolor{mycolor3}{71.2}&71.7&\textcolor{mycolor4}{\textbf{54.4}}&\textcolor{mycolor4}{\underline{54.6}}&\textcolor{mycolor4}{\underline{51.8}}&\textcolor{mycolor4}{53.2}&\textcolor{mycolor4}{66.9}&\textcolor{mycolor4}{66.2}&\textcolor{mycolor4}{62.7}&\textcolor{mycolor4}{55.8}&\textcolor{mycolor4}{\textbf{51.4}}&\textcolor{mycolor4}{\textbf{45.1}}&\textcolor{mycolor4}{\underline{49.7}}&\textcolor{mycolor4}{\underline{51.1}}&\textcolor{mycolor4}{57.1} &\textcolor{mycolor4}{53.6} &\textcolor{mycolor4}{\underline{89.1}}\\EBM (FT) & & &\textcolor{mycolor3}{52.9}&\textcolor{mycolor3}{72.2}&\textcolor{mycolor3}{72.8}&56.0&\textcolor{mycolor3}{56.2}&\textcolor{mycolor3}{53.5}&54.1&\textcolor{mycolor3}{67.8}&67.4&\textcolor{mycolor3}{64.0}&\textcolor{mycolor3}{57.0}&\textcolor{mycolor4}{\underline{52.2}}&46.5&\textcolor{mycolor3}{51.4}&\textcolor{mycolor3}{52.9}&\textcolor{mycolor3}{58.5}&\textcolor{mycolor3}{59.7} &\textcolor{mycolor4}{87.5}\\DPN & & & 51.9 & \textcolor{mycolor4}{\underline{69.2}} & \textcolor{mycolor4}{\textbf{69.6}} & \textcolor{mycolor3}{56.5} & 55.6 & 52.4 & \textcolor{mycolor4}{53.1} & \textcolor{mycolor4}{\textbf{65.5}} & \textcolor{mycolor4}{\textbf{65.0}} & \textcolor{mycolor4}{\textbf{62.2}} & \textcolor{mycolor4}{\textbf{55.5}} & \textcolor{mycolor3}{54.5} & 46.0 & \textcolor{mycolor4}{\underline{49.7}} & \textcolor{mycolor4}{51.4} & \textcolor{mycolor4}{57.2} &\textcolor{mycolor3}{63.8} &87.2\\WOODS & & &\textcolor{mycolor4}{51.4} & \textcolor{mycolor4}{69.4} & \textcolor{mycolor4}{\underline{70.0}} & \textcolor{mycolor4}{55.2} & \textcolor{mycolor4}{54.8} & \textcolor{mycolor4}{51.9} & \textcolor{mycolor4}{\underline{52.9}} & \textcolor{mycolor4}{\underline{66.2}} & \textcolor{mycolor4}{66.0} & \textcolor{mycolor4}{\underline{62.5}} & \textcolor{mycolor4}{\underline{55.6}} & \textcolor{mycolor4}{52.6} & \textcolor{mycolor4}{45.6} & \textcolor{mycolor4}{\underline{49.7}} & \textcolor{mycolor4}{51.3} & \textcolor{mycolor4}{\underline{57.0}} &\textcolor{mycolor4}{\underline{51.7}} &\textcolor{mycolor4}{88.3}\\SCONE& & &\textcolor{mycolor4}{51.6} & \textcolor{mycolor4}{69.4} & \textcolor{mycolor4}{\underline{70.0}} & \textcolor{mycolor4}{55.4} & \textcolor{mycolor4}{55.0} & \textcolor{mycolor4}{52.1} & \textcolor{mycolor4}{53.1} & \textcolor{mycolor4}{66.3} & \textcolor{mycolor4}{66.0} & \textcolor{mycolor4}{62.6} & \textcolor{mycolor4}{55.7} & 53.0 & \textcolor{mycolor4}{45.8} & 49.9 & \textcolor{mycolor4}{51.4} & \textcolor{mycolor4}{57.1}&\textcolor{mycolor4}{52.5} &\textcolor{mycolor4}{88.2}\\DUL (Ours) & & &\cellcolor{gray!20}\textcolor{mycolor4}{\textbf{51.0}} & \cellcolor{gray!20}\textcolor{mycolor4}{\textbf{69.1}} & \cellcolor{gray!20}\textcolor{mycolor4}{70.5} & \cellcolor{gray!20}\textcolor{mycolor4}{\underline{55.1}} & \cellcolor{gray!20}\textcolor{mycolor4}{\textbf{54.5}} & \cellcolor{gray!20}\textcolor{mycolor4}{\textbf{51.5}} & \cellcolor{gray!20}\textcolor{mycolor4}{\textbf{52.4}} & \cellcolor{gray!20}\textcolor{mycolor4}{\underline{66.2}} & \cellcolor{gray!20}\textcolor{mycolor4}{\underline{65.9}} & \cellcolor{gray!20}\textcolor{mycolor4}{62.6} & \cellcolor{gray!20}\textcolor{mycolor4}{55.7} & \cellcolor{gray!20}53.0 & \cellcolor{gray!20}\textcolor{mycolor4}{\underline{45.4}} & \cellcolor{gray!20}\textcolor{mycolor4}{\textbf{49.4}} & \cellcolor{gray!20}\textcolor{mycolor4}{\textbf{50.9}} & \cellcolor{gray!20}\textcolor{mycolor4}{\textbf{56.9}} &\cellcolor{gray!20}\textcolor{mycolor4}{\textbf{49.1}} &\cellcolor{gray!20}\textcolor{mycolor4}{\textbf{89.3}}\\
\bottomrule
\end{tabular}
}
\vskip -0.2in
\end{table}

\subsection{Results of different types of corruption.}
We conduct additional experiments on CIFAR10-C, CIFAR100-C and ImageNet-C with 15 different types of corruption. The results validate that the proposed method can improve the overall performance under different types of corruption.

\subsection{OOD detection results on individual datasets.}
We provide OOD detection results of DUL on each individual OOD detection test dataset in Tab.~\ref{tab:ood detection performance on each OOD dataset} and Tab.~\ref{tab:ood detection performance on each OOD dataset imagenet}, based on the checkpoint with random seed 1.

\subsection{Empirical evidence}
Here we provide empirical supports to our intuition about energy-based OOD detection regularization. We calculate the entropy of predicted distribution before and after finetuning with Energy regularization~\cite{liu2020energy}. The results show can support our claim in Section~\ref{sec:theory}.

\begin{table}[!htbp]
\vskip 0.15in
\begin{center}
\caption{Predictive entropy of predicted distribution on covariate-shifted OOD dataset before and after finetuning with energy-based OOD detection regularization~\cite{liu2020energy}.}
\label{tab:entropy}
\center
\resizebox{0.5\textwidth}{!}{
\setlength{\tabcolsep}{3.9mm}
\begin{tabular}{c|c|c|c}
\toprule                                              
$\mathcal{P}^{\rm ID}$&$\mathcal{P}^{\rm SEM}_{\rm train}$& Before  &After   \\ \midrule
CIFAR10 & ImageNet-RC & 0.11 & 1.05 \\
CIFAR10 & TIN-597 & 0.11 & 0.14 \\
CIFAR100 & ImageNet-RC & 0.92 & 3.39 \\
CIFAR100 & TIN-597 & 0.92 & 1.15
\\
       \bottomrule
\end{tabular}}
\end{center}
\end{table}

\section{Discussions}
\subsection{Math derivation}
\label{appendix:derivation}

\textbf{Differential entropy.} The proposed DUL calculate differential entropy as OOD detection measurement. Here we detail how to calculate the differential entropy of a Dirichlet distribution. The following derivation of differential entropy is taken from~\cite{malinin2018predictive}. The differential entropy of a Dirichlet parameterized by $\mathbf{\alpha}$ is calculated by
\begin{equation}
    \begin{split}
        h[p(\mu|x)]&=-\int_S p(\mu|x)\ln(p(\mu|x))d\mu
    \\
    &=\sum^K_k\ln \Gamma(\alpha_k)-\ln \Gamma(\alpha_0)-\sum^K_k(\alpha_k-1)\cdot(\psi(\alpha_k)-\psi(\alpha_0))
    \end{split}
\end{equation}
where $\alpha_0$ is the strength of Dirichlet, i.e., $\alpha_0=\sum_K \alpha_k$. $\alpha_k$ denotes the $k$-th element in $\mathbf{\alpha}$. $\Gamma$ is the Gamma function and $\psi$ is the digamma function. Here we provide a PyTorch implementation on how to calculate distribution uncertainty measured by differential entropy.

\begin{python}
def diff_entropy(alphas):
    alpha0 = torch.sum(alphas, dim=1)
    return torch.sum(
            torch.lgamma(alphas)-(alphas-1)*(torch.digamma(alphas)-torch.digamma(alpha0).unsqueeze(1)),
            dim=1) - torch.lgamma(alpha0)
            
logits = model(x)
alpha = torch.Relu(logits)+1
diff_entropy = diff_entropy(alpha)
\end{python}

We refer interested readers to~\cite{malinin2018predictive} and Gal's PhD Thesis~\cite{gal2016uncertainty} for more detailed math derivations.

\subsection{Discussion about Disparity Discrepancy}
\label{appendix:discussion}
In section~\ref{sec:theory}, we claim that a limited disparity discrepancy between test-time semantic OOD and covariant-shifted OOD is practical. Here we provide more evidence and discussion to support such a claim.

\textbf{The key challenge of OOD detection lies in identifying ID-like semantic OOD.} As mentioned in recent works~\cite{bai2023id}, effectively distinguishing between the most challenging OOD samples that are much like in-distribution (ID) data is the core challenge of OOD detection. Recent works regularize models on ID-like OOD to enhance the OOD detection performance. Since the ID-like semantic OOD samples are more difficult to be detected and more informative. For example, NTOM~\cite{chen2021atom} and POEM~\cite{ming2022poem} utilizes greedy and Thompson sampling strategies to find semantic OOD samples which are more closely to ID. ~\cite{bai2023id} proposes to explicitly discover outliers near ID by prompt learning.

\textbf{Semantic OOD and covariate OOD can be very similar.} As shown in Fig.~\ref{fig:examples}. There exists many similar samples from semantic OOD and covariate OOD in large-scale commonly used benchmarks. We borrow some examples from recent works~\cite{bitterwolf2023ninco} to show case. 

\begin{figure}[htbp]
\centering
\includegraphics[width=0.99\textwidth]{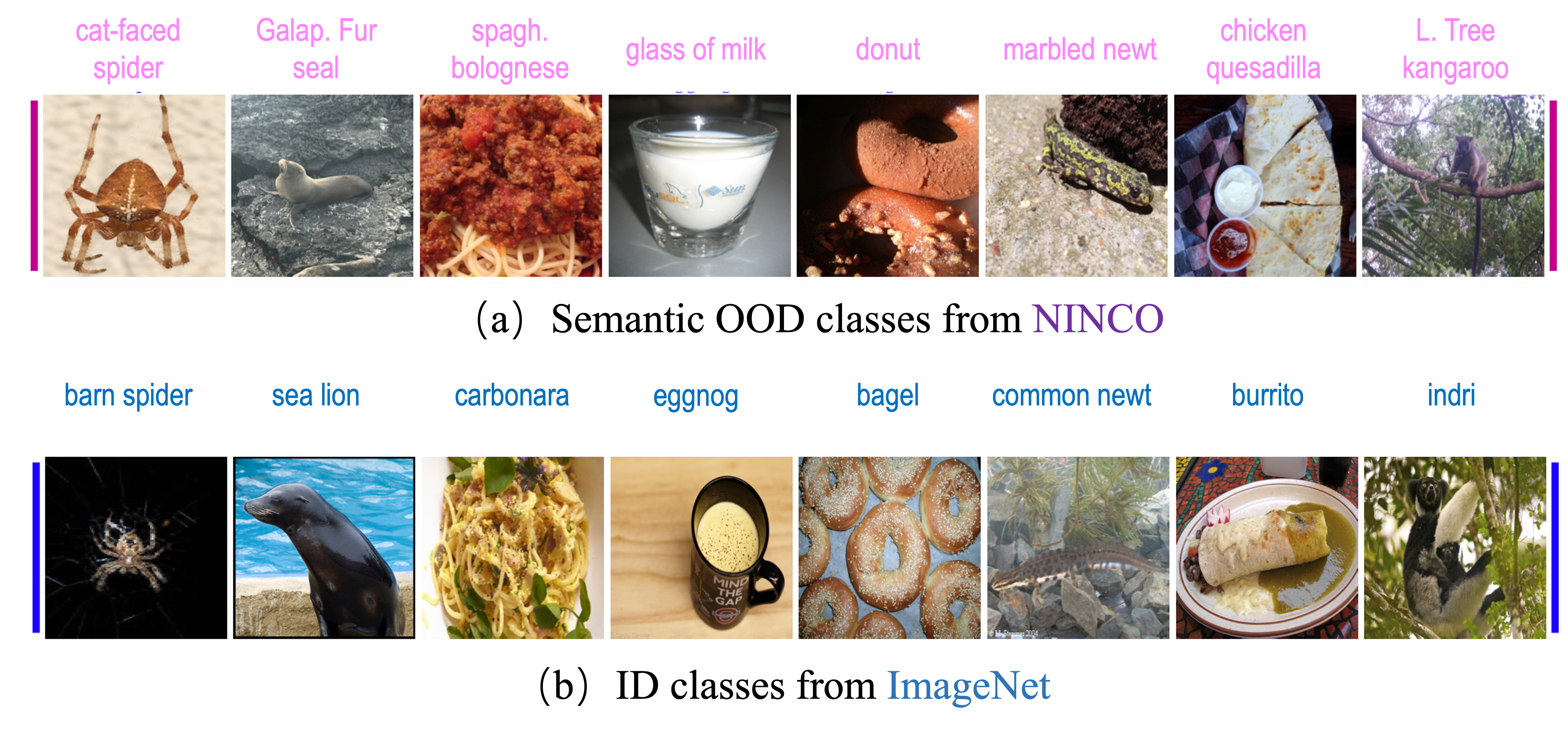}
\vspace{-3mm}
\caption{Semantic OOD samples can be very similar to ID.}
\vspace{-3mm}
\label{fig:examples}
\end{figure}

\subsection{Social Impact}
AI safety and trustworthiness are closely related to our work. This paper presents work to harmonize the conflicts between out-of-distribution detection methods and model generalization. The proposed method puts effort to enhance machine learning models for their safely deployment on out-of-distribution data, avoiding both undesirable behavior and degraded performance in challenging high-stake tasks. However, due to the bias from data used by current OOD detection benchmark, e.g., large-scale ImageNet, the ones using the proposed method need to carefully consider the selection of auxiliary outliers for safety-critical applications.

\newpage
\section*{NeurIPS Paper Checklist}

\begin{enumerate}

\item {\bf Claims}
    \item[] Question: Do the main claims made in the abstract and introduction accurately reflect the paper's contributions and scope?
    \item[] Answer: \answerYes{} 
    \item[] Justification: The abstract and introduction accurately reflect the paper's contributions and scope.
    \item[] Guidelines:
    \begin{itemize}
        \item The answer NA means that the abstract and introduction do not include the claims made in the paper.
        \item The abstract and/or introduction should clearly state the claims made, including the contributions made in the paper and important assumptions and limitations. A No or NA answer to this question will not be perceived well by the reviewers. 
        \item The claims made should match theoretical and experimental results, and reflect how much the results can be expected to generalize to other settings. 
        \item It is fine to include aspirational goals as motivation as long as it is clear that these goals are not attained by the paper. 
    \end{itemize}

\item {\bf Limitations}
    \item[] Question: Does the paper discuss the limitations of the work performed by the authors?
    \item[] Answer: \answerYes{} 
    \item[] Justification: In conclusion and Appendix.
    \item[] Guidelines:
    \begin{itemize}
        \item The answer NA means that the paper has no limitation while the answer No means that the paper has limitations, but those are not discussed in the paper. 
        \item The authors are encouraged to create a separate "Limitations" section in their paper.
        \item The paper should point out any strong assumptions and how robust the results are to violations of these assumptions (e.g., independence assumptions, noiseless settings, model well-specification, asymptotic approximations only holding locally). The authors should reflect on how these assumptions might be violated in practice and what the implications would be.
        \item The authors should reflect on the scope of the claims made, e.g., if the approach was only tested on a few datasets or with a few runs. In general, empirical results often depend on implicit assumptions, which should be articulated.
        \item The authors should reflect on the factors that influence the performance of the approach. For example, a facial recognition algorithm may perform poorly when image resolution is low or images are taken in low lighting. Or a speech-to-text system might not be used reliably to provide closed captions for online lectures because it fails to handle technical jargon.
        \item The authors should discuss the computational efficiency of the proposed algorithms and how they scale with dataset size.
        \item If applicable, the authors should discuss possible limitations of their approach to address problems of privacy and fairness.
        \item While the authors might fear that complete honesty about limitations might be used by reviewers as grounds for rejection, a worse outcome might be that reviewers discover limitations that aren't acknowledged in the paper. The authors should use their best judgment and recognize that individual actions in favor of transparency play an important role in developing norms that preserve the integrity of the community. Reviewers will be specifically instructed to not penalize honesty concerning limitations.
    \end{itemize}

\item {\bf Theory Assumptions and Proofs}
    \item[] Question: For each theoretical result, does the paper provide the full set of assumptions and a complete (and correct) proof?
    \item[] Answer: \answerYes{} 
    \item[] Justification: The paper provide the full set of assumptions and a complete (and correct) proof in Appendix.
    \item[] Guidelines:
    \begin{itemize}
        \item The answer NA means that the paper does not include theoretical results. 
        \item All the theorems, formulas, and proofs in the paper should be numbered and cross-referenced.
        \item All assumptions should be clearly stated or referenced in the statement of any theorems.
        \item The proofs can either appear in the main paper or the supplemental material, but if they appear in the supplemental material, the authors are encouraged to provide a short proof sketch to provide intuition. 
        \item Inversely, any informal proof provided in the core of the paper should be complemented by formal proofs provided in appendix or supplemental material.
        \item Theorems and Lemmas that the proof relies upon should be properly referenced. 
    \end{itemize}

    \item {\bf Experimental Result Reproducibility}
    \item[] Question: Does the paper fully disclose all the information needed to reproduce the main experimental results of the paper to the extent that it affects the main claims and/or conclusions of the paper (regardless of whether the code and data are provided or not)?
    \item[] Answer: \answerYes{} 
    \item[] Justification: Section 6 and Appendix.
    \item[] Guidelines:
    \begin{itemize}
        \item The answer NA means that the paper does not include experiments.
        \item If the paper includes experiments, a No answer to this question will not be perceived well by the reviewers: Making the paper reproducible is important, regardless of whether the code and data are provided or not.
        \item If the contribution is a dataset and/or model, the authors should describe the steps taken to make their results reproducible or verifiable. 
        \item Depending on the contribution, reproducibility can be accomplished in various ways. For example, if the contribution is a novel architecture, describing the architecture fully might suffice, or if the contribution is a specific model and empirical evaluation, it may be necessary to either make it possible for others to replicate the model with the same dataset, or provide access to the model. In general. releasing code and data is often one good way to accomplish this, but reproducibility can also be provided via detailed instructions for how to replicate the results, access to a hosted model (e.g., in the case of a large language model), releasing of a model checkpoint, or other means that are appropriate to the research performed.
        \item While NeurIPS does not require releasing code, the conference does require all submissions to provide some reasonable avenue for reproducibility, which may depend on the nature of the contribution. For example
        \begin{enumerate}
            \item If the contribution is primarily a new algorithm, the paper should make it clear how to reproduce that algorithm.
            \item If the contribution is primarily a new model architecture, the paper should describe the architecture clearly and fully.
            \item If the contribution is a new model (e.g., a large language model), then there should either be a way to access this model for reproducing the results or a way to reproduce the model (e.g., with an open-source dataset or instructions for how to construct the dataset).
            \item We recognize that reproducibility may be tricky in some cases, in which case authors are welcome to describe the particular way they provide for reproducibility. In the case of closed-source models, it may be that access to the model is limited in some way (e.g., to registered users), but it should be possible for other researchers to have some path to reproducing or verifying the results.
        \end{enumerate}
    \end{itemize}

\item {\bf Open access to data and code}
    \item[] Question: Does the paper provide open access to the data and code, with sufficient instructions to faithfully reproduce the main experimental results, as described in supplemental material?
    \item[] Answer: \answerYes{} 
    \item[] Justification: See Appendix and supplemental material.
    \item[] Guidelines:
    \begin{itemize}
        \item The answer NA means that paper does not include experiments requiring code.
        \item Please see the NeurIPS code and data submission guidelines (\url{https://nips.cc/public/guides/CodeSubmissionPolicy}) for more details.
        \item While we encourage the release of code and data, we understand that this might not be possible, so “No” is an acceptable answer. Papers cannot be rejected simply for not including code, unless this is central to the contribution (e.g., for a new open-source benchmark).
        \item The instructions should contain the exact command and environment needed to run to reproduce the results. See the NeurIPS code and data submission guidelines (\url{https://nips.cc/public/guides/CodeSubmissionPolicy}) for more details.
        \item The authors should provide instructions on data access and preparation, including how to access the raw data, preprocessed data, intermediate data, and generated data, etc.
        \item The authors should provide scripts to reproduce all experimental results for the new proposed method and baselines. If only a subset of experiments are reproducible, they should state which ones are omitted from the script and why.
        \item At submission time, to preserve anonymity, the authors should release anonymized versions (if applicable).
        \item Providing as much information as possible in supplemental material (appended to the paper) is recommended, but including URLs to data and code is permitted.
    \end{itemize}

\item {\bf Experimental Setting/Details}
    \item[] Question: Does the paper specify all the training and test details (e.g., data splits, hyperparameters, how they were chosen, type of optimizer, etc.) necessary to understand the results?
    \item[] Answer: \answerYes{} 
    \item[] Justification: The paper provides above details in Section 6 and Appendix.
    \item[] Guidelines:
    \begin{itemize}
        \item The answer NA means that the paper does not include experiments.
        \item The experimental setting should be presented in the core of the paper to a level of detail that is necessary to appreciate the results and make sense of them.
        \item The full details can be provided either with the code, in appendix, or as supplemental material.
    \end{itemize}

\item {\bf Experiment Statistical Significance}
    \item[] Question: Does the paper report error bars suitably and correctly defined or other appropriate information about the statistical significance of the experiments?
    \item[] Answer: \answerYes{} 
    \item[] Justification: Appendix C.
    \item[] Guidelines:
    \begin{itemize}
        \item The answer NA means that the paper does not include experiments.
        \item The authors should answer "Yes" if the results are accompanied by error bars, confidence intervals, or statistical significance tests, at least for the experiments that support the main claims of the paper.
        \item The factors of variability that the error bars are capturing should be clearly stated (for example, train/test split, initialization, random drawing of some parameter, or overall run with given experimental conditions).
        \item The method for calculating the error bars should be explained (closed form formula, call to a library function, bootstrap, etc.)
        \item The assumptions made should be given (e.g., Normally distributed errors).
        \item It should be clear whether the error bar is the standard deviation or the standard error of the mean.
        \item It is OK to report 1-sigma error bars, but one should state it. The authors should preferably report a 2-sigma error bar than state that they have a 96\% CI, if the hypothesis of Normality of errors is not verified.
        \item For asymmetric distributions, the authors should be careful not to show in tables or figures symmetric error bars that would yield results that are out of range (e.g. negative error rates).
        \item If error bars are reported in tables or plots, The authors should explain in the text how they were calculated and reference the corresponding figures or tables in the text.
    \end{itemize}

\item {\bf Experiments Compute Resources}
    \item[] Question: For each experiment, does the paper provide sufficient information on the computer resources (type of compute workers, memory, time of execution) needed to reproduce the experiments?
    \item[] Answer: \answerYes{} 
    \item[] Justification: In Appendix C.
    \item[] Guidelines:
    \begin{itemize}
        \item The answer NA means that the paper does not include experiments.
        \item The paper should indicate the type of compute workers CPU or GPU, internal cluster, or cloud provider, including relevant memory and storage.
        \item The paper should provide the amount of compute required for each of the individual experimental runs as well as estimate the total compute. 
        \item The paper should disclose whether the full research project required more compute than the experiments reported in the paper (e.g., preliminary or failed experiments that didn't make it into the paper). 
    \end{itemize}
    
\item {\bf Code Of Ethics}
    \item[] Question: Does the research conducted in the paper conform, in every respect, with the NeurIPS Code of Ethics \url{https://neurips.cc/public/EthicsGuidelines}?
    \item[] Answer: \answerYes{} 
    \item[] Justification: The authors have carefully reviewed the NeurIPS Code of Ethics.
    \item[] Guidelines:
    \begin{itemize}
        \item The answer NA means that the authors have not reviewed the NeurIPS Code of Ethics.
        \item If the authors answer No, they should explain the special circumstances that require a deviation from the Code of Ethics.
        \item The authors should make sure to preserve anonymity (e.g., if there is a special consideration due to laws or regulations in their jurisdiction).
    \end{itemize}

\item {\bf Broader Impacts}
    \item[] Question: Does the paper discuss both potential positive societal impacts and negative societal impacts of the work performed?
    \item[] Answer: \answerYes{} 
    \item[] Justification: In Appendix.
    \item[] Guidelines:
    \begin{itemize}
        \item The answer NA means that there is no societal impact of the work performed.
        \item If the authors answer NA or No, they should explain why their work has no societal impact or why the paper does not address societal impact.
        \item Examples of negative societal impacts include potential malicious or unintended uses (e.g., disinformation, generating fake profiles, surveillance), fairness considerations (e.g., deployment of technologies that could make decisions that unfairly impact specific groups), privacy considerations, and security considerations.
        \item The conference expects that many papers will be foundational research and not tied to particular applications, let alone deployments. However, if there is a direct path to any negative applications, the authors should point it out. For example, it is legitimate to point out that an improvement in the quality of generative models could be used to generate deepfakes for disinformation. On the other hand, it is not needed to point out that a generic algorithm for optimizing neural networks could enable people to train models that generate Deepfakes faster.
        \item The authors should consider possible harms that could arise when the technology is being used as intended and functioning correctly, harms that could arise when the technology is being used as intended but gives incorrect results, and harms following from (intentional or unintentional) misuse of the technology.
        \item If there are negative societal impacts, the authors could also discuss possible mitigation strategies (e.g., gated release of models, providing defenses in addition to attacks, mechanisms for monitoring misuse, mechanisms to monitor how a system learns from feedback over time, improving the efficiency and accessibility of ML).
    \end{itemize}
    
\item {\bf Safeguards}
    \item[] Question: Does the paper describe safeguards that have been put in place for responsible release of data or models that have a high risk for misuse (e.g., pretrained language models, image generators, or scraped datasets)?
    \item[] Answer: \answerNA{} 
    \item[] Justification: The paper poses no such risks.
    \item[] Guidelines:
    \begin{itemize}
        \item The answer NA means that the paper poses no such risks.
        \item Released models that have a high risk for misuse or dual-use should be released with necessary safeguards to allow for controlled use of the model, for example by requiring that users adhere to usage guidelines or restrictions to access the model or implementing safety filters. 
        \item Datasets that have been scraped from the Internet could pose safety risks. The authors should describe how they avoided releasing unsafe images.
        \item We recognize that providing effective safeguards is challenging, and many papers do not require this, but we encourage authors to take this into account and make a best faith effort.
    \end{itemize}

\item {\bf Licenses for existing assets}
    \item[] Question: Are the creators or original owners of assets (e.g., code, data, models), used in the paper, properly credited and are the license and terms of use explicitly mentioned and properly respected?
    \item[] Answer: \answerYes{} 
    \item[] Justification: The license and terms of use explicitly mentioned and properly respected.
    \item[] Guidelines:
    \begin{itemize}
        \item The answer NA means that the paper does not use existing assets.
        \item The authors should cite the original paper that produced the code package or dataset.
        \item The authors should state which version of the asset is used and, if possible, include a URL.
        \item The name of the license (e.g., CC-BY 4.0) should be included for each asset.
        \item For scraped data from a particular source (e.g., website), the copyright and terms of service of that source should be provided.
        \item If assets are released, the license, copyright information, and terms of use in the package should be provided. For popular datasets, \url{paperswithcode.com/datasets} has curated licenses for some datasets. Their licensing guide can help determine the license of a dataset.
        \item For existing datasets that are re-packaged, both the original license and the license of the derived asset (if it has changed) should be provided.
        \item If this information is not available online, the authors are encouraged to reach out to the asset's creators.
    \end{itemize}

\item {\bf New Assets}
    \item[] Question: Are new assets introduced in the paper well documented and is the documentation provided alongside the assets?
    \item[] Answer: \answerNA{} 
    \item[] Justification: The paper does not release new assets.
    \item[] Guidelines:
    \begin{itemize}
        \item The answer NA means that the paper does not release new assets.
        \item Researchers should communicate the details of the dataset/code/model as part of their submissions via structured templates. This includes details about training, license, limitations, etc. 
        \item The paper should discuss whether and how consent was obtained from people whose asset is used.
        \item At submission time, remember to anonymize your assets (if applicable). You can either create an anonymized URL or include an anonymized zip file.
    \end{itemize}

\item {\bf Crowdsourcing and Research with Human Subjects}
    \item[] Question: For crowdsourcing experiments and research with human subjects, does the paper include the full text of instructions given to participants and screenshots, if applicable, as well as details about compensation (if any)? 
    \item[] Answer: \answerNA{} 
    \item[] Justification: The paper does not involve crowdsourcing nor research with human subjects.
    \item[] Guidelines:
    \begin{itemize}
        \item The answer NA means that the paper does not involve crowdsourcing nor research with human subjects.
        \item Including this information in the supplemental material is fine, but if the main contribution of the paper involves human subjects, then as much detail as possible should be included in the main paper. 
        \item According to the NeurIPS Code of Ethics, workers involved in data collection, curation, or other labor should be paid at least the minimum wage in the country of the data collector. 
    \end{itemize}

\item {\bf Institutional Review Board (IRB) Approvals or Equivalent for Research with Human Subjects}
    \item[] Question: Does the paper describe potential risks incurred by study participants, whether such risks were disclosed to the subjects, and whether Institutional Review Board (IRB) approvals (or an equivalent approval/review based on the requirements of your country or institution) were obtained?
    \item[] Answer: \answerNA{} 
    \item[] Justification: The paper does not involve crowdsourcing nor research with human subjects.
    \item[] Guidelines:
    \begin{itemize}
        \item The answer NA means that the paper does not involve crowdsourcing nor research with human subjects.
        \item Depending on the country in which research is conducted, IRB approval (or equivalent) may be required for any human subjects research. If you obtained IRB approval, you should clearly state this in the paper. 
        \item We recognize that the procedures for this may vary significantly between institutions and locations, and we expect authors to adhere to the NeurIPS Code of Ethics and the guidelines for their institution. 
        \item For initial submissions, do not include any information that would break anonymity (if applicable), such as the institution conducting the review.
    \end{itemize}
\end{enumerate}

\end{document}